%% file: main.tex
\documentclass[11pt]{article}
\usepackage{authblk}
\usepackage{graphicx} 
\usepackage[title]{appendix}
\usepackage{algorithm}
\usepackage{algpseudocode}
\usepackage[margin=1in]{geometry}
\usepackage{hyperref}
\hypersetup{
    colorlinks=true,
    linkcolor=blue,
    citecolor=ToCgreen,
    filecolor=blue,
    urlcolor=blue
}
\usepackage{multirow}
\usepackage{cprotect}

\newcommand{\TG}{\mathsf{TGeom}}
\newcommand{\Unif}{\mathsf{Unif}}
\newcommand{\cN}{\mathsf{N}}
\newcommand{\rate}{{\cal R}}

\usepackage{ifthen}
\newboolean{journal}
\setboolean{journal}{true}

\title{Differentially Private Optimization with Sparse Gradients}
\author[1]{Badih Ghazi} 
\author[1,2]{Crist\'obal Guzm\'an}
\author[1]{Pritish Kamath} 
\author[1]{Ravi Kumar} 
\author[1]{Pasin Manurangsi}
\affil[1]{Google Research} 
\affil[2]{Institute for Mathematical and Computational Engineering, Faculty of Mathematics and School of Engineering, Pontificia Universidad Cat\'olica de Chile}
\date{}

\usepackage{amsmath,amsfonts,amsthm,amssymb,array,dsfont}
\usepackage{enumitem}
\usepackage{xcolor}
\usepackage[square,numbers]{natbib}
\usepackage{cleveref}

\newcommand{\prn}[1]{\left( #1 \right)}
\newcommand{\sq}[1]{\left [ #1 \right ]}

\def\ddefloop#1{\ifx\ddefloop#1\else\ddef{#1}\expandafter\ddefloop\fi}
\def\ddef#1{\expandafter\def\csname #1#1\endcsname{\ensuremath{\mathbb{#1}}}}
\ddefloop ABCDEFGHIJKLMNOPQRSTUVWXYZ\ddefloop
\def\ddef#1{\expandafter\def\csname c#1\endcsname{\ensuremath{\mathcal{#1}}}}
\ddefloop ABCDEFGHIJKLMNOPQRSTUVWXYZ\ddefloop
\def\ddef#1{\expandafter\def\csname h#1\endcsname{\ensuremath{\hat{#1}}}}
\ddefloop xyz\ddefloop
\def\ddef#1{\expandafter\def\csname t#1\endcsname{\ensuremath{\tilde{#1}}}}
\ddefloop xyz\ddefloop

\newcommand{\eps}{\varepsilon}

\newcommand{\conv}{\mathrm{conv}}
\newcommand{\Lap}{\mathsf{Lap}}
\newcommand{\sparsevec}{{\cal S}_{s}^d}

\newcommand{\Lip}{L}
\newcommand{\smooth}{H}
\newcommand{\gap}{{\Gamma}}
\newcommand{\optpop}{x^{\ast}({\cal D})}

\newcommand{\opterm}{x^{\ast}(S)}
\newcommand{\optregerm}{x_{\lambda}^{\ast}(S)}
\newcommand{\optregermp}{x_{\lambda}^{\ast}(S^{\prime})}

\DeclareMathOperator*{\argmin}{\mathrm{argmin}}

\newtheorem{theorem}{Theorem}[section]
\newtheorem{proposition}[theorem]{Proposition}
\newtheorem{lemma}[theorem]{Lemma}

\newtheorem{remark}[theorem]{Remark}
\newtheorem{fact}[theorem]{Fact}
\theoremstyle{definition}
\newtheorem{definition}[theorem]{Definition}

\providecommand{\Comments}{0}  
\ifnum\Comments=1
\usepackage[colorinlistoftodos,prependcaption,textsize=scriptsize]{todonotes}
\paperwidth=\dimexpr\paperwidth + 2cm\relax
\marginparwidth=\dimexpr\marginparwidth + 1.7cm\relax
\else
\usepackage[disable]{todonotes}
\fi
\newcommand{\mytodo}[1]{\ifnum\Comments=1{#1}\fi}

\newcommand{\tableoftodos}{\ifnum\Comments=1 \listoftodos[Comments/To Do's] \fi}

\definecolor{Gred}{RGB}{219, 50, 54}
\definecolor{Ggreen}{RGB}{60, 186, 84}
\definecolor{Gblue}{RGB}{72, 133, 237}
\definecolor{Gyellow}{RGB}{247, 178, 16}
\definecolor{ToCgreen}{RGB}{0, 128, 0}
\definecolor{myGold}{RGB}{231,141,20}
\definecolor{myBlue}{rgb}{0.19,0.41,.65}
\definecolor{myPurple}{RGB}{175,0,124}

\begin{document}

\maketitle

\input{DP_SCO_Sparse_Gradients/s0-abstract}

\input{DP_SCO_Sparse_Gradients/s1-introduction}

\input{DP_SCO_Sparse_Gradients/s2-prelims}

\input{DP_SCO_Sparse_Gradients/s3-mean-estimation-upper-bounds}

\input{DP_SCO_Sparse_Gradients/s4-mean-estimation-lower-bounds}

\input{DP_SCO_Sparse_Gradients/s6-dp-erm-nonconvex}

\input{DP_SCO_Sparse_Gradients/s7-boosting}

\input{DP_SCO_Sparse_Gradients/s5-dp-erm-sco-output-perturbation}

\section*{Acknowledgements}
C.G.'s research was
partially supported by
INRIA Associate Teams project, ANID FONDECYT 1210362 grant, ANID Anillo ACT210005 grant, and National Center for Artificial Intelligence CENIA FB210017, Basal ANID.

\bibliography{DP_SCO_Sparse_Gradients/bibliography}


\section*{Appendix}

\input{DP_SCO_Sparse_Gradients/a1-appendix}

\end{document}

%% file: DP_SCO_Sparse_Gradients/s0-abstract.tex
\begin{abstract}
Motivated by applications of large embedding models, we study differentially private (DP) optimization problems under sparsity of {\em individual} gradients.  We start with  new near-optimal bounds for the classic mean estimation problem but with sparse data, improving upon existing algorithms particularly for the high-dimensional regime. The corresponding lower bounds are based on a novel block-diagonal construction used in combination with existing DP mean estimation lower bounds.
Next, we obtain pure- and approximate-DP algorithms with almost optimal rates 
for stochastic convex optimization with sparse gradients; the former represents the first nearly dimension-independent rates for this problem. Furthermore, by introducing novel analyses of bias-reduction in mean estimation and randomly-stopped biased SGD we  obtain nearly dimension independent rates for near stationary points for the empirical risk in nonconvex settings under approximate-DP.
\end{abstract}

%% file: DP_SCO_Sparse_Gradients/s1-introduction.tex
\section{Introduction}

The pervasiveness of personally sensitive data in machine learning applications (e.g.,~advertising, public policy, and healthcare) has led to the major concern of protecting users' data from their exposure. When releasing or deploying these trained models, differential privacy (DP) offers a rigorous and quantifiable guarantee on the privacy exposure risk \citep{Dwork:2014}.

Consider neural networks whose inputs have categorical features with large vocabularies. These features can be modeled
using embedding tables; namely, for a feature that takes $K$ distinct values, we create trainable parameters $w_1, \ldots, w_K \in \RR^k$, and use $w_a$ as input to the neural network when the corresponding input feature is $a$. A natural outcome of such models is that the per-example gradients are guaranteed to be sparse; when the input feature is $a$, then only the gradient with respect to $w_a$ is non-zero. Given the prevalence of sparse gradients in practical deep learning applications, GPUs/TPUs that are optimized to leverage gradient sparsity are commercially offered and widely used in industry~\cite{nvidia2019accelerating,wang2022merlin,cloudtpu,jouppi2023tpu}. To leverage gradient sparsity, recent practical work has considered DP stochastic optimization with {\em sparse gradients} for large embedding models for different applications including  recommendation systems, natural language processing, and ads modeling \citep{Zhang:2021, Ghazi:2023}.

Despite its relevance and promising empirical results, there is limited understanding of the theoretical limits of DP learning under gradient sparsity. This gap motivates our work. 

\subsection{Our Results}

We initiate the study of DP optimization under gradient sparsity. 
More precisely, we consider a stochastic optimization (SO) problem, $\min\{F_{\cD}(x):x\in {\cX}\}$, where ${\cX}\subseteq \RR^d$ is a convex set, and $F_{\cD}(x)=\mathbb{E}_{z\sim {\cD}}[f(x,z)]$, with $f(\cdot,z)$ enjoying some regularity properties, 
and ${\cD}$ is a probability measure supported on a set ${\cZ}$. Our main assumption is gradient sparsity: for an integer $0\leq s\leq d$,
\[
\forall x\in \cX, \, z\in \cZ \ : \qquad \|\nabla f(x,z)\|_0 \leq s\,,
\]
where $\|y\|_0$ denotes the number of nonzero entries of $y$.  
We also study empirical risk minimization (ERM), where given a dataset $S=(z_1,\ldots,z_n)$ we aim to minimize $F_S(x) := \frac1n \sum_{i\in[n]} f(x,z_i)$.

Our results unearth three regimes of accuracy rates for the above setting: 
(i) the small dataset size regime where the optimal rate is constant, (ii) the large dataset size where the optimal rates are polynomial in the dimension  and (iii) an intermediate dataset size regime characterized by a new high-dimensional rate\footnote{We will generally refer to high-dimensional or nearly dimension-independent rates indistinguishably, meaning more precisely that the rates scale poly-logarithmically with the dimension.}  
(see \Cref{tab:sparse_mean_estimation} and \Cref{tab:DP_SO_rates}, for precise rates). 
These results imply in particular that even for high-dimensional models, this problem is tractable under gradient sparsity.  Without sparsity, these poly-logarithmic rates are unattainable in the light of known lower bounds \citep{Bassily:2014}.

\input{DP_SCO_Sparse_Gradients/Table}

In \Cref{sec:MeanEstimationUB}, we start with the fundamental task of $\ell_2$-mean estimation with sparse data (which reduces to ERM with sparse linear losses \citep{Bassily:2014}). Here, we obtain new  upper bounds (see \Cref{tab:sparse_mean_estimation}). These rates are obtained by adapting the projection mechanism \citep{Nikolov:2013},  
with a convex relaxation that makes our algorithms efficient. Note that for pure-DP, even our large dataset rate of $\sqrt{sd}/(\varepsilon n)$ can be substantially smaller than the dense pure-DP rate of $d/(\varepsilon n)$ \citep{Bassily:2014}, whenever $s\ll d$. For approximate-DP we also obtain a sharper upper bound by solving an $\ell_1$-regression problem of a noisy projection of the empirical mean over a random subspace. Its analysis combines ideas from compressed sensing \citep{Wojtaszczyk:2010} with sparse approximation via the Approximate Carath\'eodory Theorem \citep{Pisier:1981}. 

In \Cref{sec:LowerBounds}, we prove lower bounds that show the near-optimality of our algorithms. For pure-DP, we obtain a new lower bound of $\Omega(s\log(d/s)/(n\varepsilon))$, which is based on a packing of sparse vectors. While this lower bound looks weaker than the standard $\Omega(d/(n\varepsilon))$ lower bound based on dense packings \citep{Hardt:2010,Bassily:2014}, we design a novel bootstrapping via a block diagonal construction where each block contains a sparse lower bound as above. This, together with a padding argument \citep{Bassily:2014}, yields  lower bounds for the three regimes of interest. For approximate-DP, we also use the block diagonal bootstrapping, where this time the blocks use classical fingerprinting codes in dimension $s$ \citep{Bassily:2014,Steinke:2016}. Our approximate-DP lower bounds, however, have a gap  
of $\ln(d/s)^{1/4}$ in the high-dimensional regime; we conjecture that the aforementioned compressed sensing-based upper bound is tight. 

In \Cref{sec:bias_reduction}, we study DP-ERM with sparse gradients, under approximate-DP. We propose the use of stochastic gradient (SGD) with a mean estimation gradient oracle based on the results in  \Cref{sec:MeanEstimationUB}.  This technique yields nearly-tight bounds in the convex case (similar to first row of \Cref{tab:DP_SO_rates}), and for the nonconvex case the stationarity rates are nearly dimension independent (last row of \Cref{tab:DP_SO_rates}).  The main challenge here is the {\em bias in mean estimation}, which dramatically deteriorates the rates of SGD. 
Hence we propose a bias-reduction method inspired by the simulation literature \citep{Blanchet:2015}. This technique uses a random batch size in an exponentially increasing schedule and a telescopic estimator of the gradient which---used in conjunction with our DP mean estimation methods---provides a stochastic first-order oracle that attains bias similar to the one of a full-batch algorithm, with moderately bounded variance.  Note that using the full-batch in this case would lead to polynomially weaker rates; in turn, our method leverages the batch randomization to conduct a more careful privacy accounting based on subsampling and the fully-adaptive properties of DP \citep{Whitehouse:2023}. The introduction of random batch sizes and the random evolution of the privacy budget leads to various challenges into analyzing the performance of SGD. First, we analyze a {\em randomly stopped  method}, where the stopping time dictated by the privacy budget. Noting that the standard SGD analysis bounds the cumulative regret, which is a submartingale, we carry out this analysis by integrating ideas from submartingales and stopping times \cite{Rosenthal:2006}. Second, this analysis only yields the desired rates {\em with constant probability}. Towards high-probability results, we leverage a private model selection \cite{Liu:2019} based on multiple runs of randomly-stopped SGD that exponentially boosts the success probability (details in \Cref{sec:boosting}).

In \Cref{sec:Output_Perturbation}, we study further DP-SO and DP-ERM algorithms for the convex case. Our algorithms are based on regularized output perturbation with an $\ell_{\infty}$ projection post-processing step. While this projection step is rather unusual, its role is clear from the analysis: it leverages the $\ell_{\infty}$ bounds of noise addition, which in conjunction with convexity provides an error guarantee that also leverages the gradient sparsity. This algorithm is nearly-optimal for approximate DP. For pure-DP, the previous algorithm requires an additional smoothness assumption, hence we propose a second algorithm  
based on the exponential mechanism \citep{McSherry:2007} run over a net of suitably sparse vectors. Neither of the pure-DP algorithms matches the lower bound for mean estimation (the gap in the exponent of the rate is of $1/6$), but they attain the first nearly dimension-independent rates for this problem.

\subsection{Related Work}

DP optimization is an extensively studied topic for over a decade (see \cite{Bassily:2014,Bassily:2019,Feldman:2020}, and the references therein).  
In this field, some works have highlighted the role of {\em model sparsity} (e.g.,~using sparsity-promoting $\ell_1$-ball constraints) in near-dimension independent excess-risk rates for DP optimization, both for ERM and SCO \citep{Jain:2014,Talwar:2014,Talwar:2015,Asi:2021,Bassily:2021,Cai:2021,Cai:2020}.  
These settings are unrelated to ours, as sparse predictors are typically related to dense gradients.

Another proposed assumption to mitigate the impact of dimension in DP learning is that gradients lie (approximately) in a low dimensional subspace \citep{Zhou:2021,Kairouz:2021,Li:2022,Lee:2024} or where dimension is substituted by a bound on the trace of the Hessian of the loss \citep{Ma:2022}. These useful results are unfortunately not applicable to our setting of interest, 
as we are interested in arbitrary gradient sparsity patterns for different datapoints.

Substantially less studied is the role of gradient sparsity. Closely related to our work,  \citet{Zhang:2021} studied approximate DP-ERM under gradient sparsity, with some stronger assumptions. 
Aside from an additional $\ell_{\infty}$ bound on individual gradients, the following {\em partitioning sparsity assumption} is imposed: the dataset $S$ can be uniformly partitioned into subsets $S_1,\ldots, S_m$ with a uniform gradient sparsity bound: for all $k\in[m]$ and $x\in {\cX}$,
$\big\|\sum_{z\in S_k} \nabla f(x,z)\big\|_0 \leq c_1$.   
The work shows poly-logarithmic in the dimension rates, for both convex and nonconvex settings\ifthenelse{\boolean{journal}}{: for convex objectives the rate is $\frac{\mbox{\footnotesize polylog}(d)}{\varepsilon \sqrt{n}}$, and in the nonconvex case is  $\frac{\mbox{\footnotesize polylog}(d)}{\sqrt{\varepsilon} n^{1/4}}$.}{.} 
Our results only assume individual gradient sparsity, so on top of being more general, they are also faster and provably nearly optimal in the convex case.  
Another relevant work is \citet{Ghazi:2023}, which studies the computational and utility benefits for DP with sparse gradients in neural networks with embedding tables.  
With the caveat that variable selection on stochastic gradients is performed at the level of {\em contributing buckets} (i.e.~rows of the embedding table), rather than on gradient coordinates, this work shows substantial improvements on computational efficiency 
and also on the resulting utility.  

In 
\citet{Asi:2021bias}, bias-reduction is used to mitigate the regularization bias in SCO. While they also borrow inspiration from \cite{Blanchet:2015}, both their techniques and scope of their work are unrelated to ours.

\subsection{Future Directions}

\ifthenelse{\boolean{journal}}{
Our work provides a number of structural results on DP optimization with sparse gradients. We hope these results inspire further research in this area, thus we highlight some open questions.}{}

\ifthenelse{\boolean{journal}}{}{We present some of the main open questions and future directions of this work.}
First, we conjecture that for approximate-DP mean estimation---similarly to the pure-DP case---a lower bound $\Omega\big(\sqrt{s\log(d/s)\ln(1/\delta)}/[n\varepsilon]\big)$ should exist; such construction could be bootstrapped with a block-diagonal dataset for a tight lower bound  (\Cref{lem:block_diagonal_bootstrap}). 
Second, for pure DP-SCO, we believe there should exist an algorithm attaining rates analog to those for mean estimation. Unfortunately, most of variants of output perturbation (including phasing  \citep{Feldman:2020,Asi:2021,Asi:2021adapting}) cannot attain such rates. From a practical perspective, the main open question is whether our rates may be attained without prior knowledge of $s$; note that all our mean estimation algorithms (which carries over to our optimization results) depend crucially on knowledge of this parameter. While we can treat $s$ as a hyperparameter, it would be highly beneficial to design algorithms that automatically adapt to it.

We believe our bias-reduction is of broader interest. For example, \cite{Kamath:2023,Nikolov:2023} have shown strong negative results about bias in DP mean estimation. While similar lower bounds may hold for sparse estimation, bias-reduction allows us to amortize this error within an iterative method, preventing error accumulation. 

Finally, there is no evidence of our nonconvex rate being optimal. \ifthenelse{\boolean{journal}}{
While it is known that in the dense case noisy SGD is not optimal \citep{Arora:2023}, known acceleration approaches are based on variance reduction, which appears to be incompatible with our bias-reduction.}{In this vein, we should remark that even in the dense case the optimal stationarity rates are still open \citep{Arora:2023}.}

%% file: DP_SCO_Sparse_Gradients/Table.tex
\newcommand{\myresult}[1]
{\colorbox{blue!5}{$\textcolor{black}{#1}$}}

{\renewcommand{\arraystretch}{1.8}

\begin{table}[t]
    \centering
    \small
    \begin{tabular}{|c|c c|c p{0.1\linewidth}|}
    \hline
        {\small\bf Setting}  
          & \multicolumn{2}{c}{{\small\bf Upper bound}}  & \multicolumn{2}{|c|}{{\small\bf Lower bound}} \\
      \hline
     \hline      
      \hspace{-0.2cm} {\small $\eps$-DP} \hspace{-0.2cm}
        &
        $1\wedge\myresult{\hspace{-0.1cm}\sqrt{\frac{s\ln d}{\eps n}}} \wedge \myresult{\hspace{-0.1cm}\frac{\sqrt{sd}}{\eps n}\hspace{-0.1cm}}$ \hspace{-0.4cm}
        & {\small (Thm. \ref{thm:proj_mech_pure_DP})} & $1\wedge\myresult{\hspace{-0.1cm}\sqrt{\frac{s\ln(d/(\eps n))}{\eps n}}\hspace{-0.1cm}} \wedge \myresult{\hspace{-0.1cm}\frac{\sqrt{sd}}{\eps n}\hspace{-0.1cm}}$ \hspace{-0.3cm} & {\small (Thm. \ref{thm:LB_pure_DP_sparse_bootstrap})}\hspace{-0.2cm} \\
    \hline
        \hspace{-0.2cm}{\small $(\eps,\delta)$-DP} \hspace{-0.2cm}
        & $1 \wedge \myresult{\hspace{-0.1cm}\frac{(s\ln(d/s)\ln(1/\delta))^{1/4}}{\sqrt{\eps n}}\hspace{-0.1cm}} \wedge \frac{\sqrt{d\ln(1/\delta)}}{\eps n} $ \hspace{-0.4cm} 
        & {\small (Thm.~\ref{thm:Compressed_Sensing_UB_sparse_Mean_Est})} \hspace{-0.2cm}
        & $1 \wedge \myresult{\hspace{-0.1cm}\frac{(s\ln(1/\delta))^{1/4}}{\sqrt{\eps n}}\hspace{-0.1cm}} \wedge \frac{\sqrt {d \ln(1/\delta)}}{\eps n}$ \hspace{-0.4cm}
        &{\small (Thm. \ref{thm:LB_approx_DP_sparse})} \hspace{-0.2cm}\\
    \hline
    \end{tabular}
    \vspace*{3pt}
    \caption{Rates for DP mean estimation with sparse data of unit $\ell_2$-norm.  Bounds stated for constant success/failure probability, resp. We use $a \wedge b$ to denote $\min(a, b)$. New results  $\myresult{\mathrm{highlighted}}$. 
\vspace{-2mm}}\label{tab:sparse_mean_estimation}
\end{table}}

{\renewcommand{\arraystretch}{1.8}
\begin{table}[t]
    \centering
    \small
    \begin{tabular}{|c|c|cc|c|}
    \hline
          {\small\bf Setting} &  {\small\bf Guarantee} & \multicolumn{2}{c|}{{\small\bf \shortstack[c]{New Upper bound\\ (sparse)}}} & {\small\bf \shortstack[c]{\mbox{}\\[1mm] Upper bound\\(non-sparse)}} \\ 
    \hline
    \hline
        \multirow{2}{*}{{\small $(\eps,\delta)$-DP}} &  {\small Cvx.~ERM}
        & $\myresult{ \frac{(s \ln(d)\ln(1/\delta))^{1/4}}{\sqrt{\eps n}} }\wedge \rate_{\varepsilon,\delta}$
        \ \  & \hspace{-0.5cm} {\small  \ifthenelse{\boolean{journal}}{(Thm.~\ref{thm:output-pert-improved})}{(Thm.~\ref{thm:biased_SGD_merged}, \ref{thm:output-pert-improved})}
        }
        & $\rate_{\varepsilon,\delta}$\\
        \cline{2-5}
        & {\small SCO} &
        $\myresult{\frac{(s\ln(d)\ln(1/\delta))^{1/4}}{\sqrt{\eps n}}}\wedge\rate_{\varepsilon,\delta}+\frac{1}{\sqrt n}$
        & {\small(Thm. \ref{thm:output_pert_approx_DP_SCO})}
        & $\rate_{\varepsilon,\delta}+\frac{1}{\sqrt n}$\\
    \hline
        {\renewcommand{\arraystretch}{1}\multirow{2}{*}{\small{
        \begin{tabular}{c}{\small $\eps$-DP}\end{tabular}}}}
        &  {\small Cvx.~ERM}
        & \multicolumn{2}{r|}{
            $\myresult{\Big( \frac{s\ln(d)}{\eps n} \Big)^{1/3}}\wedge \rate_{\varepsilon}$ \quad\ \ 
            {\small (Thm. \ref{thm:output-pert-improved}, \ref{thm:sparse_exp_mech})}
        } 
        & $\rate_{\varepsilon}$\\
    \cline{2-5}
        & {\small SCO}
        & {\small $\myresult{\Big(\frac{s\ln(d)}{\eps n}\Big)^{1/3}}\wedge\rate_{\varepsilon}+\frac{1}{\sqrt n}$}
        & {\small (Thm. \ref{thm:output_pert_approx_DP_SCO})}
        & $\rate_{\varepsilon}+\frac{1}{\sqrt n}$\\
    \hline
        {\small $(\eps,\delta)$-DP}
        & {\small Emp.~Grad.~Norm}
        & $\myresult{\frac{(s\ln(d/s)\ln^3(1/\delta))^{1/8}}{(\eps n)^{1/4}}}\wedge \big(\rate_{\varepsilon,\delta}\big)^{2/3}$
        & {\small 
        \ifthenelse{\boolean{journal}}{(Thm. \ref{thm:nonconvex_const_proba})}{(Thm.~\ref{thm:biased_SGD_merged})} }
        & $\big(\rate_{\varepsilon,\delta}\big)^{2/3}$\\
    \hline
    \end{tabular}
    \vspace*{3pt}
    \caption{Rates for DP optimization with sparse gradients, compared to best-existing upper bounds in the non-sparse case. Bounds stated for constant success probability. Function parameters and polylog$(n)$ factors  omitted. Above $\rate_{\varepsilon,\delta}={\sqrt{d \ln(1/\delta)}}/{[\eps n]}$ and $\rate_{\varepsilon}={d}/{[\eps n]}$. Improvements  $\myresult{
    \mathrm{highlighted}}$.\vspace{-3mm}
    } 
    \label{tab:DP_SO_rates}
\end{table}
}

%% file: DP_SCO_Sparse_Gradients/s2-prelims.tex
\section{Notation and Preliminaries}

\label{sec:NotationPrelims}

In this work, $\|\cdot\|=\|\cdot\|_2$ is the standard Euclidean norm on $\RR^d$. We will also make use of $\ell_p$-norms, where $\|x\|_p := \big(\sum_{j\in[d]}|x_j|^p\big)^{1/p}$ for $1\leq p\leq \infty$. For $p=0$, we use the notation $\|x\|_0=|\{j\in[d]:x_j\neq 0\}|$, i.e.,~the cardinality of the support of $x$.  
We denote the $r$-radius ball centered at $x$ of the $p$-norm in $\RR^d$ by ${\cal B}_p^d(x,r):=\{y\in \RR^d: \|y-x\|_p\leq r\}$.  
Given $s\in[d]$ and $\Lip>0$, the set of {\em $s$-sparse vectors} is (the scaling factor $L$ is omitted in the notation for brevity) 
\begin{equation} \label{eqn:sparse_vectors}
\sparsevec := \{x\in\RR^d:\, \|x\|_0\leq s,\, \|x\|_2\leq \Lip\}.
\end{equation}
Note that Jensen's inequality implies: if $\|x\|_0\leq s$ and $1\leq p<q\leq \infty$, then $\|x\|_p \le s^{1/p-1/q} \|x\|_q$.
\begin{remark}\label{rem:upper-bound-l1}
The upper bound results 
in this paper hold even if we replace the set $\sparsevec$ of sparse vectors by the strictly larger $\ell_1$ ball ${\cal B}_1^d(0, L \sqrt{s})$ \ifthenelse{\boolean{journal}}{(inclusion follows from the aforementioned inequality)}{\!}. Note that while our upper bounds extend to the $\ell_1$ assumption above, our lower bounds work under the original sparsity assumption. 
\end{remark}

Let $f:\cX\times\cZ \mapsto \RR$ be a loss function. The function evaluation $f(x,z)$ represents the loss incurred by hypothesis $x\in \cX$ on datapoint $z\in \cZ$. In {\em stochastic optimization} (SO), we consider a data distribution ${\cal D}$, and our goal is to minimize the expected loss under this distribution
\begin{equation} \label{eqn:SO} \tag{SO}
\textstyle \min_{x\in \cX} \Big\{ F_{\cD}(x):=\mathbb{E}_{z\sim{\cal D}}[f(x,z)]  \Big\}.
\end{equation}
Throughout, we use $\optpop$ 
to denote an optimal solution to \eqref{eqn:SO}, which we assume exists. \ifthenelse{\boolean{journal}}{We note in passing that existence of such optimal solution is guaranteed when ${\cal X}$ is compact, but otherwise additional work is needed to assert such existence; on the other hand, if multiple optimal solutions exist, we can choose an arbitrary one (e.g., a minimal norm solution).}{} 
In the {\em empirical risk minimization} (ERM) problem, we consider sample datapoints $S=(z_1,\ldots,z_n)$ and our goal is to minimize the empirical error with respect to the sample
\begin{equation} \label{eqn:ERM} \tag{ERM}
\textstyle \min_{x\in\cX} \Big\{ F_{S}(x):=\frac1n\sum_{i\in[n]} f(x,z_i) \Big\}.
\end{equation}
We denote by $\opterm $ an arbitrary optimal solution to \eqref{eqn:ERM}, which we assume exists. Even when $S$ is drawn i.i.d.~from ${\cal D}$, solutions (or optimal values) of \eqref{eqn:SO} and \eqref{eqn:ERM} do not necessarily coincide.

We present the definition of differential privacy (DP), deferring useful properties and examples to \Cref{app:auxiliary_results}. Let $\cZ$ be a sample space, and $\cX$ an output space. 
A dataset is a tuple $S\in \cZ^n$, and datasets   
$S,S^{\prime} \in \cZ^n$ are \emph{neighbors} (denoted as $S\simeq S^{\prime}$) if they differ in only one of their entries.

\begin{definition}[Differential Privacy]
Let $\cA:\cZ^n\mapsto \cX$. We say that $\cA$ is \emph{$(\eps,\delta)$-(approximately) differentially private (DP)} if for every pair  $S\simeq S^{\prime}$, we have for all $\cE\subseteq \cX$ that
\ifthenelse{\boolean{journal}}{
\[ \Pr[\cA(S) \in \cE] ~\leq~ e^{\eps} \cdot \Pr[\cA(S^{\prime}) \in \cE]+\delta. \]
}{$\Pr[\cA(S) \in \cE] ~\leq~ e^{\eps} \cdot \Pr[\cA(S^{\prime}) \in \cE]+\delta.$ }
When $\delta=0$, we say that $\cA$ is $\eps$-DP or pure-DP.
\end{definition}

\ifthenelse{\boolean{journal}}{\input{DP_SCO_Sparse_Gradients/Proofs/facts_noise_addition}}{}

%% file: DP_SCO_Sparse_Gradients/Proofs/facts_noise_addition.tex
The privacy and accuracy of some of the perturbation based methods we use to privatize our algorithms are based on the following simple facts (see e.g. \cite{Dwork:2014}).

\begin{fact}[Laplace \& Gaussian mechanisms]\label{fact:privacy}
For all $g:{\cal Z}^n\mapsto \mathbb{R}^d$
\begin{enumerate}
\item[(a)] If $\ell_1$-sensitivity of $g$ is bounded, i.e.,~$\Delta_1^g:=\sup_{S\simeq S^{\prime}}\|g(S)-g(S^{\prime})\|_1<+\infty$, then ${\cal A}_{\Lap}^g(S) := g(S)+\xi$ where $\xi \sim \Lap^{\otimes d}(\Delta_1^g/\varepsilon)$ is $\varepsilon$-DP.
\item[(b)] If $\ell_2$-sensitivity of $g$ is bounded, i.e.,~$\Delta_2^g:=\sup_{S\simeq S^{\prime}}\|g(S)-g(S^{\prime})\|_2<+\infty$, then ${\cal A}_{\cN}^g(S):=g(S)+\xi$, where $\xi \sim \cN \big(0,\sigma^2I\big)$ for $\sigma \ge \frac{\Delta_2^g \sqrt{2 \log (1.25 / \delta)}}{\varepsilon}$ is $(\varepsilon,\delta)$-DP.
\end{enumerate}
\end{fact}

\begin{fact} [Laplace \& Gaussian concentration]\label{fact:accuracy}
Let $\sigma>0$ and $0<\beta<1$.
\begin{enumerate}
\item[(a)] For $\xi\sim \Lap(\sigma)^{\otimes d}$: (i)  $\|\xi\|_{\infty}\lesssim \sigma \log(d/\beta)$ holds with probability $1-\beta$, and (ii) $\|\xi\|_2\lesssim \sigma\sqrt{d}\log(d/\beta)$ holds with probability $1-\beta$.
\item[(b)] For $\xi\sim \cN(0,\sigma^2 I)$, (i) $\|\xi\|_{\infty}\lesssim \sigma\sqrt{\log(d/\beta)}$ holds with probability $1-\beta$, (ii) $\|\xi\|_2\lesssim \sigma\big( \sqrt{d}+\sqrt{\log(1/\beta)}\big)$ holds with probability $1-\beta$, and (iii) $\mathbb{E}\|\xi\|_2^2=d\sigma^2$.
\end{enumerate}
\end{fact}

%% file: DP_SCO_Sparse_Gradients/s3-mean-estimation-upper-bounds.tex
\section{Upper Bounds for DP Mean Estimation with Sparse Data}
\label{sec:MeanEstimationUB}
We first study DP mean estimation with sparse data. Our first result is that the projection mechanism \cite{Nikolov:2013} is nearly optimal, both for pure- and approximate-DP. In our case, we interpret the marginals on each of the $d$ dimensions as the queries of interest: this way, the $\ell_2$-error on private query answers corresponds exactly to the $\ell_2$-norm estimation error. A key difference to the approach in \cite{Nikolov:2013} and related works is that we project the noisy answers onto the set ${\cal K}:={\cal B}_1^d(0,\Lip\sqrt{s})$, which is a (coarse) convex relaxation of $\conv(\sparsevec)$. This is crucial to make our algorithm efficiently implementable. 
\ifthenelse{\boolean{journal}}{}{Due to space limitations, proofs from this section have been deferred to \Cref{app:Mean_Estimation}. }

\begin{algorithm}
\cprotect\caption{%
\texttt{Projection\_Mechanism$(\bar{z}(S),\varepsilon,\delta,n)$}%
}\label[algorithm]{alg:proj}
\begin{algorithmic}
\Require Vector $\bar z(S)=\frac1n\sum_{i=1}^nz_i$ from dataset $S\in (\sparsevec)^n$; $\eps,\delta\geq0$, privacy parameters\vspace{1mm}
\State $\tz = \bar z(S) + \xi$, with $\xi\sim\begin{cases}
\Lap(\sigma)^{\otimes d} &\text{with } \sigma=\big(\frac{2\Lip\sqrt{s}}{n\eps}\big) \text{ if } \delta = 0\,,\\[1mm]
\cN(0,\sigma^2I) &\text{with } \sigma^2=\frac{8\Lip^2\ln(1.25/\delta)}{(n\eps)^2} \text{ if } \delta > 0\,.\\
\end{cases}$
\State \Return $\hz = \argmin\{\|z-\tz\|_2:  z\in {\cal K}\}$, 
where ${\cal K}:={\cal B}_1^d(0,\Lip\sqrt{s})$
\end{algorithmic}
\end{algorithm}

\begin{lemma} \label[lemma]{lem:proj-main} 
In \Cref{alg:proj}, it holds that $\|\hz - \bar{z}(S)\|_2 \leq \sqrt{2\Lip \|\xi\|_{\infty} \sqrt{s}}$, almost surely. 
\end{lemma}

\ifthenelse{\boolean{journal}}{\input{DP_SCO_Sparse_Gradients/Proofs/lem01-proj-main}}{}

We now provide the privacy and accuracy guarantees of \Cref{alg:proj}.

\begin{theorem} \label[theorem]{thm:proj_mech_pure_DP}
For $\delta=0$, \Cref{alg:proj} is $\eps$-DP, and with probability $1 - \beta$:
\begin{align*}
\textstyle \|\hz - \bar z(S)\|_2 \lesssim \Lip \cdot \min\left\{\frac{\sqrt{s d} \ln(d/\beta)}{n \eps}, \sqrt{\frac{s \ln(d/\beta)}{n \eps}} \right\}.
\end{align*}
\end{theorem}

\ifthenelse{\boolean{journal}}
{\input{DP_SCO_Sparse_Gradients/Proofs/thm01proj_mech_pure_DP}}
{}

\begin{theorem} \label[theorem]{thm:proj_mech_approx_DP}
For $\delta>0$, \Cref{alg:proj} is  $(\eps, \delta)$-DP, and with probability $1 - \beta$:
\begin{align*}
\textstyle \|\hz - \bar z(S)\|_2 \lesssim \Lip \cdot \min\left\{\frac{(\sqrt{d}+\sqrt{\log(1/\beta)})
\sqrt{\ln(1/\delta)} }{n \eps} , \frac{(s \log(1/\delta) \log(d/\beta))^{1/4}}{\sqrt{n \eps}} \right\}.
\end{align*}
\end{theorem}

\ifthenelse{\boolean{journal}}
{\input{DP_SCO_Sparse_Gradients/Proofs/thm2_proj_mech_approx_DP}}
{}

\ifthenelse{\boolean{journal}}{
\subsection{Sharper Approximate DP Mean Estimation via Compressed Sensing} 

\input{DP_SCO_Sparse_Gradients/compressed_sensing_mean_estimation}

}{
\paragraph{Sharper Upper Bound via Compressed Sensing} 
In \Cref{sec:compressed_sensing_mean_estimation} we propose a faster mean estimation approximate DP algorithm. Its rate nearly matches the lower bound we will prove in \Cref{thm:LB_pure_DP_sparse}.  
We believe that this rate is essentially optimal. This algorithm projects the data average into a low dimensional subspace (via a random projection matrix), and uses compressed sensing to recover a noisy version of this projection: this way, noise provides privacy, which is further boosted by the random projection, and the accuracy follows from an application of the stable and noisy recovery properties of compressed sensing \citep{Wojtaszczyk:2010}, together with the Approximate Carath\'eodory Theorem.
}

%% file: DP_SCO_Sparse_Gradients/Proofs/lem01-proj-main.tex
\begin{proof}
From the properties of the Euclidean projection, we have
\begin{align} \label{eq:proj-angle}
\langle \hz-\bar z(S), \hz -\tz \rangle \leq 0.
\end{align}
Hence,
\begin{align*}
\|\hz - \bar z(S)\|_2^2 &= \left<\hz - \bar z(S), \hz - \tz\right> + \left<\hz - \bar z(S), \xi\right>
~\overset{\eqref{eq:proj-angle}}{\leq}~ \left<\hz - \bar z(S), \xi\right> \\
&\leq 2 \cdot \max_{u \in {\cal K}} \left<u, \xi\right>
~\leq~ 2 \cdot \max_{u \in {\cal K}} \|u\|_1 \cdot \|\xi\|_\infty
~=~ 2\Lip \|\xi\|_\infty \sqrt{s},
\end{align*}
where we used the fact that $\mbox{conv}
(\sparsevec)\subseteq {\cal K}$.
\end{proof}

%% file: DP_SCO_Sparse_Gradients/Proofs/thm01proj_mech_pure_DP.tex
\begin{proof}
First, the privacy follows from the $\ell_1$-sensitivity bound of the empirical mean 
\[
\textstyle \Delta_1= \sup_{S\simeq S^{\prime}}\|\bar z(S)-\bar z(S^{\prime})\|_1=\frac1n\sup_{z,z^{\prime}\in \sparsevec}\|z-z^{\prime}\|_1 \leq \frac{2\Lip\sqrt s}{n},
\]
together with \Cref{fact:privacy}(a).

For the accuracy, the first term follows from \Cref{fact:accuracy}(a)-(ii), 
and the fact that Euclidean projection does not increase the $\ell_2$-estimation error, and the second term follows from 
\Cref{lem:proj-main} with the fact that $\|\xi\|_\infty \le O\left(\frac{L\sqrt{s}}{n \eps} \cdot \log(d/\beta)\right)$  
holds with probability at least $1-\beta$, by \Cref{fact:accuracy}(a)-(i).
\end{proof}

%% file: DP_SCO_Sparse_Gradients/Proofs/thm2_proj_mech_approx_DP.tex
\begin{proof}
The privacy guarantee follows from the $\ell_2$-sensitivity bound of the empirical mean, $\Delta_2=\frac{2\Lip}{n}$, together with \Cref{fact:privacy}(b).
For the accuracy, the first term in the minimum follows from \Cref{fact:accuracy}(b)-(ii),  
and the fact that Euclidean projection does not increase the $\ell_2$-estimation error.  
The second term follows from \Cref{lem:proj-main} and \Cref{fact:accuracy}(b)-(ii).
\end{proof}

%% file: DP_SCO_Sparse_Gradients/compressed_sensing_mean_estimation.tex
\label{sec:compressed_sensing_mean_estimation}

We propose \Cref{alg:Gaussian_ell1_Rec}, a more accurate method for approximate-DP mean estimation based on compressed sensing \citep{Wojtaszczyk:2010}. The precise improvements relate to reducing the $\log(d)$ factor to $\log(d/s)$, and a faster rate dependence on the confidence $\beta$. The idea is that for sufficiently high dimensions a small number of random measurements suffices to estimate a {\em noisy and approximately sparse signal}. These properties follow from existing results in compressed sensing, which provide guarantees based on the $\ell_2$-norm of the noise, and the best sparse approximation in the $\ell_2$-norm (known as $\ell_2\text{-}\ell_2$-stable and noisy recovery) \citep{Wojtaszczyk:2010}. We will exploit such robustness in two ways: regarding the noise robustness, this property is used in order to perturb our measurements, which will certify the privacy; on the other hand, the approximate recovery property is used to find a sparser approximation of our empirical mean. As the approximation is only used for analysis, we can resort on the Approximate Caratheodory Theorem to certify the existence of a sparse vector whose sparsity increases more moderately with $n$ than the empirical average \citep{Pisier:1981}.

An interesting feature of this algorithm is that $\ell_1$-minimization promotes sparse solutions, and thus we expect our output to be approximately sparse: this is not a feature that we particularly exploit, but it may be relevant for computational and memory considerations. Furthermore, note that the $\ell_1$-minimization problem does not require exact optimality for the privacy guarantee, hence approximate solvers can be used without compromising the privacy.

\input{DP_SCO_Sparse_Gradients/Proofs/alg02_compressed_sensing}


\begin{proof}
First, if $d<m\ln^2 m$, Algorithm \ref{alg:Gaussian_ell1_Rec} is $(\varepsilon,\delta)$-DP by privacy of Gaussian noise addition and the post-processing property of DP. Moreover, its (high-probability and second moment) accuracy guarantees are direct from \Cref{fact:accuracy}.

Next, if $d\geq m\ln^2 m$, we start with the privacy analysis. Let $S\simeq S^{\prime}$ and suppose they only differ in their $i$-th entry. We note that due to our choice of $m$, $A$ is an approximate restricted isometry with probability $1-3\exp\{-cm\}$ \citep{Candes:2005} (where $c$ is the same as in the theorem statement); in particular, letting $K\eqsim \frac{n\varepsilon}{\sqrt{s\ln(d/s)\ln(1/\delta)}}$, we have that for all $v\in \mathbb{R}^d$ which is $(sK)$-sparse
\[ \frac12 \|v\|_2 \leq \|Av\|_2 \leq \frac32 \|v\|_2. \]
Hence, due to our assumption on $\delta$, the event above has probability at least $1-\delta/2$, and therefore
\[ \|A(\bar z-\bar z^{\prime})\|_2 = \frac1n \|A(z_i-z_i^{\prime})\|_2 \leq \frac{3\Lip}{n}, \]
where we used the fact that $z_i-z_i^{\prime}$ is $(2s)$-sparse. 
We conclude by the choice of $\sigma^2$ that $A\bar z+\xi$ is $(\varepsilon,\delta)$-DP, and thus $\tilde z$ is $(\varepsilon,\delta)$-DP by postprocessing.

We now proceed to the accuracy guarantee. By \cite[Theorem 3.6 (b)]{Wojtaszczyk:2010}, under the same event as stated above (which has probability $1-\delta/2$) we have
\[ \|\hat z-\bar z\|_2 \lesssim \|\xi\|_2+\inf_{z:\,\|z\|_0\leq sK} \|z-\bar z\|_2 .\]
For the first term, we use Gaussian norm concentration to guarantee that with probability $1-\beta$,
\[ \|\xi\|_2 
\lesssim \big(\sqrt{m}+\sqrt{\ln(1/\beta)}\big)\sigma 
\lesssim  \Big(\sqrt{Ks\ln(d/s)}+\sqrt{\ln\big(\frac{1}{\beta}\big)}\Big)\frac{\Lip\sqrt{\ln(1/\delta)}}{n\varepsilon} .\]
For the second term, by the Approximate Car\'atheodory Theorem \citep{Pisier:1981}, the infimum above is upper bounded by $O(\Lip/\sqrt{K})$; for this, note that $\bar z$ lies in the convex hull of $\sparsevec$. Given our choice of $K$, we have that, with probability $1-\delta/2-\beta$ 
\[ \|\hat z-\bar z\|_2 \lesssim \Lip\Big(\frac{[s\ln(d/s)\ln(1/\delta)]^{1/4}}{\sqrt{n\varepsilon}}+\frac{\sqrt{\ln(1/\beta)\ln(1/\delta)}}{n\varepsilon}\Big) .\]

We conclude by providing the second moment estimate, by a simple tail integration argument. 
First, by the law of total probability, and letting ${\cal E}$ be the event of $A$ being an approximate restricted isometry,
\[ \mathbb{E}\|\hat z-\bar z\|_2^2 \leq \mathbb{E}[\|\hat z-\bar z\|_2^2|{\cal E}] + 9\Lip^2\delta, \]
where we also used that $\|\hat z\|_2\leq 2\Lip$ and $\|\bar z\|_2\leq \Lip$, almost surely. 
Now, conditionally on ${\cal E}$, we have that  
letting $\alpha\eqsim \Lip\frac{[s\ln(d/s)\ln(1/\delta)]^{1/4}}{\sqrt{n\varepsilon}}$   
(below $c>0$ is an absolute constant),
\begin{align*}
\mathbb{E}[\|\hat z-\bar z\|_2^2|{\cal E}] 
&= \int_0^{\infty} \mathbb{P}\Big[ \|\hat z-\bar z\|_2 \geq u\Big] (2u) du\\
&\leq \frac{\alpha^2}{2}+\int_0^{\infty} \mathbb{P}\Big[ \|\hat z-\bar z\|_2-\alpha \geq \tau\Big] 2(\alpha+\tau) d\tau\\
&\leq \frac{\alpha^2}{2}+\int_0^{\infty}  2\exp\Big\{-\frac{c(n\varepsilon)^2}{\Lip^2\ln(1/\delta)}\tau^2\Big\}(\alpha+\tau) d\tau\\
&\lesssim \frac{\alpha^2}{2}+2\alpha \Lip\frac{\sqrt{\ln(1/\delta)}}{n\varepsilon} + \Lip^2\frac{\ln(1/\delta)}{(n\varepsilon)^2}  \\
&\lesssim \alpha^2,
\end{align*}
where in the second inequality we used the previous high-probability upper bound (here $c>0$ is an absolute constant), and in the last step we used that $n\varepsilon>\sqrt{\ln(1/\delta)}$.
Finally, by our assumptions on $\delta$, $9L^2\delta\lesssim \alpha^2$, this concludes the proof.
\end{proof}

%% file: DP_SCO_Sparse_Gradients/Proofs/alg02_compressed_sensing.tex
\begin{algorithm}
\cprotect\caption{%
\texttt{Gaussian $\ell_1$-Recovery}%
$(\bar z(S), \varepsilon,\delta,n)$}
\label{alg:Gaussian_ell1_Rec}
\begin{algorithmic}
\Require $\bar z(S)=\frac{1}{n}\sum_{i\in[n]}z_i\in\mathbb{R}^d$ from dataset $S\in ({\cal S}_{s,d})^n$; privacy parameters $\varepsilon,\delta>0$
\State $m\eqsim n\varepsilon\sqrt{\frac{s\ln(d/s)}{\ln(1/\delta)}}$
\State \Return $\hat z =
\begin{cases}
\bar z(S)+\xi \text{, where } \xi \sim \cN(0,\sigma^2I_{d\times d}) \text{ and } \sigma^2=\frac{8\Lip^2\ln(1.25/\delta)}{(n\eps)^2} \text{,}& \hspace{-2.5cm} \text{ if } d<m\ln^2m, \\[1mm]
\tilde z \cdot \mathds{1}{\{\|\tilde z\|_2\leq 2\Lip\}} \text{, where } \tilde z=\arg\min\{\|z\|_1: Az=b\}\text{, } A\sim(\cN(0,\frac1m))^{m\times d}, &\,\\[1mm]
\phantom{\tilde z \cdot \mathds{1}{\{\|\tilde z\|\leq 2\Lip\}} \text{,}} \hspace{-2cm} b=A\bar{z}(S)+\xi \text{ and } \xi\sim \cN(0,\sigma^2I_{m\times m}) \text{ with } \sigma^2=\frac{18\Lip^2\ln(2.5/\delta)}{(n\varepsilon)^2}, \hspace{1cm} \text{else}
& 
\end{cases}
$
\end{algorithmic}
\end{algorithm}

\begin{theorem} \label{thm:Compressed_Sensing_UB_sparse_Mean_Est}
If $6\exp\{-cm\}\leq \delta< \frac{s\ln(d/s)}{m^2}$  
(where $c>0$ is a constant) and $0<\eps\leq 1$, then \Cref{alg:Gaussian_ell1_Rec} is $(\eps,\delta)$-DP, and with probability $1-\delta/2-\beta$, 
\begin{equation} \label{eqn:high_proba_compressed_sensing}
\ifthenelse{\boolean{journal}}{}{\textstyle}
\|\hz-\bar z(S)\|_2 \lesssim \Lip \min\Big\{ \frac{(\sqrt{d}+\sqrt{\ln(1/\beta)}) \sqrt{\ln(1/\delta)}}{n\eps}, \frac{ (s\ln(d/s)\ln(1/\delta))^{1/4}}{\sqrt{n\eps}}+\frac{\sqrt{\ln(1/\beta)\ln(1/\delta)}}{n\eps} \Big\}. \end{equation}
Moreover, we have the following second moment estimate,
\[ \mathbb{E}[\|\hz-\bar z\|_2^2] \lesssim \Lip^2 \min\Big\{ \frac{d\ln(1/\delta)}{(n \eps)^2},\frac{\sqrt{s\ln(d/s)\ln(1/\delta)}}{n \eps} \Big\}.\]
\end{theorem}

%% file: DP_SCO_Sparse_Gradients/s4-mean-estimation-lower-bounds.tex
\section{Lower Bounds for DP Mean Estimation with Sparse Data}

\label{sec:LowerBounds}

We provide matching lower bounds to those from \Cref{sec:MeanEstimationUB}. Moreover, although the stated lower bounds are for mean estimation, known reductions imply analogous lower bounds for DP-ERM and DP-SCO \citep{Bassily:2014, Bassily:2019}. First, for pure-DP we provide a packing-type construction based on sparse vectors. This construction is used in a novel block-diagonal construction, which provides the right low/high-dimensional transition. On the other hand, for approximate-DP, a block diagonal 
reduction with  existing fingerprinting codes \citep{Bun:2014, Steinke:2016}, suffices to obtain lower bounds that exhibit a nearly tight low/high-dimensional transition. For simplicity, we consider the case of $\Lip=1$, i.e.~$\sparsevec=\{z\in\RR^d:\, \|z\|_0\leq s, \|z\|_2\leq 1\}$; it is easy to see that any lower bound scales linearly in $\Lip$.
\ifthenelse{\boolean{journal}}{}{We defer  proofs from this section to \Cref{app:LowerBounds}.}

\subsection{Lower Bounds for Pure-DP}

Our main lower bound for pure-DP mechanisms is as follows.

\begin{theorem} \label[theorem]{thm:LB_pure_DP_sparse_bootstrap}
Let $\eps>0$ and $s<d/2$. Then the empirical mean estimation problem over $\sparsevec$  
satisfies 
\[
\textstyle
\inf\limits_{\cA\,:\,\eps\mbox{-\footnotesize DP}}\ 
\sup\limits_{S\in (\sparsevec)^n}
\mathbb{P}\left[\|\cA(S) - \bar z(S) \|_2
~\gtrsim \min\Big\{1, \sqrt{\frac{s \log\prn{d/[\eps n]}}{\eps n}}, \frac{\sqrt{sd}}{\eps n}\Big\} \right] \gtrsim 1. 
\]
\end{theorem}
The statement above---as well as those which follow---should be read as ``for all DP algorithms ${\cal A}$, there exists a dataset $S$, such that the mean estimation error is lower bounded by $\alpha(n,d,\varepsilon,\delta)$ with probability at least $\beta(n,d,\varepsilon,\delta)$'' (where in this case $\alpha\gtrsim \min\Big\{1, \sqrt{\frac{s \log\prn{d/[\eps n]}}{\eps n}}, \frac{\sqrt{sd}}{\eps n}\Big\}$ and $\beta\gtrsim 1$).

We also introduce a strengthening of the worst-case lower bound, based on hard distributions. 
\begin{definition}
We say that a probability $\mu$ over ${\cal Z}^n$ induces an $(\alpha,\beta)$-{\em distributional lower bound} for $(\varepsilon,\delta)$-DP mean estimation if
$ \inf_{{\cal A}:\,(\varepsilon,\delta)\text{-DP}} \mathbb{P}_{S\sim \mu, {\cal A}}\big[ \|{\cal A}(S)-\bar{z}(S)\|_2 \geq \alpha \big] \geq \beta. $
\end{definition}
Note this type of lower bound readily implies a worst-case lower bound. On the other hand, while the existence of hard distributions follows by the existence of hard datasets (by Yao's minimax principle), we provide explicit constructions of these distributions, for the sake of clarity. 
\ifthenelse{\boolean{journal}}{
For the pure case, in \Cref{lem:packing_DP_LB} we show that an adaptation of known packing-type lower bounds \citep{Hardt:2010} provides distributional lower bounds for mean estimation, and for the approximate case, we note that fingerprinting distributions induce distributional lower bounds,  e.g.~\citep{Bun:2017,Cai:2021,Kamath:2020}.}{}

\Cref{thm:LB_pure_DP_sparse_bootstrap} follows by combining the two results that we provide next. First, and our main technical innovation in the sparse case is a block-diagonal dataset bootstrapping construction, which turns a low-dimensional lower bound into a high-dimensional one. \ifthenelse{\boolean{journal}}{While we state this result for norm-bounded sparse data, any family where low-dimensional instances can be embedded into the high-dimensional ones by padding with zeros would enjoy a similar property.}{}

\begin{lemma}[Block-Diagonal Lower Bound Bootstrapping] \label[lemma]{lem:block_diagonal_bootstrap}
Let $n_0, t\in\mathbb{N}$. Let $\mu$ be a distribution over $({\cal S}_s^t)^{n_0}$ that induces an $(\alpha_0,\rho_0)$-distributional lower bound for $(\varepsilon,\delta)$-DP mean estimation.  
Then, for any $d\geq t$, $n\geq n_0$ and $K \leq \min\big\{ \frac{n}{n_0}, \frac{d}{t} \big\}$, there exists $\tilde\mu$ over $(\sparsevec)^n$ that induces an $(\alpha,\rho)$-distributional lower bound for $(\varepsilon,\delta)$-DP mean estimation, where 
$\alpha\gtrsim \frac{\alpha_0n_0}{n}\sqrt{\rho_0 K}$ and $\rho\geq 1-\exp(-\rho_0/8)$.
\end{lemma}
\ifthenelse{\boolean{journal}}{
\input{DP_SCO_Sparse_Gradients/Proofs/lem07block_diagonal_bootstrap}
}{}
The above result needs as input a base lower bound. Here,  packing-based constructions suffice. 

\begin{theorem} \label[theorem]{thm:LB_pure_DP_sparse}
Let $\varepsilon>0$ and $s<d/2$. Then there exists a 
$(\alpha,\rho)$-distributional lower bound for $\varepsilon$-DP mean estimation over $(\sparsevec)^n$ with $\alpha\gtrsim \min\left\{1, \frac{s \log(d/s)}{\eps n}\right\}$ and $\rho=1/2$.
\end{theorem}

\ifthenelse{\boolean{journal}}{
\input{DP_SCO_Sparse_Gradients/Proofs/thm08LB_pure_DP_sparse}

\input{DP_SCO_Sparse_Gradients/Proofs/thm04LB_pure_DP_sparse_bootstrap}
}{}

\subsection{Lower Bounds for Approximate-DP}

While the lower bound for the approximate-DP case is similarly based on the block-diagonal reduction, its base lower bound follows more directly from the dense case. 

\begin{theorem} \label[theorem]{thm:LB_approx_DP_sparse}
Let $\eps\in(0,1]$, $2^{-o(n)}\leq \delta \leq \frac{1}{n^{1+\Omega(1)}}$.  
Then the 
empirical mean estimation problem over $\sparsevec$ satisfies
\[\textstyle
\inf\limits_{\cA \,:\,(\eps,\delta)\mbox{\text{\footnotesize-DP}}}
\sup\limits_{S\in (\sparsevec)^n} \mathbb{P}\left[ \|{\cA}(S)-\bar z(S)\|_2 \gtrsim  \min\Big\{1,\frac{[s\ln(1/\delta)]^{1/4}}{\sqrt{n\eps}},\frac{\sqrt{d \ln(1/\delta)}}{n\eps} \Big\} \right] \gtrsim 1. 
\]
\end{theorem}

\ifthenelse{\boolean{journal}}{
\input{DP_SCO_Sparse_Gradients/Proofs/thm06LB_approx_DP_sparse}
}{}

%% file: DP_SCO_Sparse_Gradients/Proofs/lem07block_diagonal_bootstrap.tex
\begin{proof}
Consider an $n\times d$ data matrix $D$ whose rows correspond to datapoints of a dataset $S$, and whose columns correspond to their $d$ features. We will indistinctively refer to $S$ or $D$ as needed (these are equivalent representations of a dataset). This data matrix will be comprised of $K$ diagonal blocks, $D_1,\ldots,D_K$; in particular, outside of these blocks, the matrix has only zeros. These blocks are sampled i.i.d.~from the hard distribution $\mu$ given by hypothesis. Denote $\tilde \mu$ the law of $D$. \\
Let now $\bar{z}_k(D_k)\in\mathbb{R}^t$ be the mean (over rows) of dataset $D_k$. Then, the mean (over rows) of dataset $D$ is given by
\ifthenelse{\boolean{journal}}{
\[ \bar{z}(D)=\frac{n_0}{n} \big[\bar{z}_1(D_1)\big|\ldots\big|\bar{z}_K(D_K)\big],\]
}{
$\bar{z}(D)=\frac{n_0}{n} \big[\bar{z}_1(D_1)\big|\ldots\big|\bar{z}_K(D_K)\big]$, }
where $[z_1|\ldots|z_K]\in\mathbb{R}^d$ denotes the concatenation of $z_1,\ldots,z_K$ (note that if $K<d/t$, then the concatenation above needs to be padded with $(d-tK)$-zeros, which we omit for simplicity).\\
Let ${\cal A}$ be an $(\eps,\delta)$-DP algorithm, and let ${\cal A}_k$ its output on the $k$-th block variables, then
\[\|{\cal A}(D)-\bar z(D)\|_2^2=\sum_{k=1}^K \Big\|{\cal A}_k(D)-\frac{n_0}{n}\bar{z}_k(D_k)\Big\|_2^2=\frac{n_0^2}{n^2} \sum_{k=1}^K \big\|\frac{n}{n_0}{\cal A}_k(D)-\bar{z}_k(D_k)\Big\|_2^2.\]
Let now ${\cal B}_k(D):=\frac{n}{n_0}{\cal A}_k(D)$, and note it is $(\varepsilon,\delta)$-DP w.r.t.~$D_k$ (as it is DP w.r.t.~$D$); further, by the independence of $D_1,\ldots,D_K$, we can condition on $(D_h)_{h\neq k}$, to conclude that the squared $\ell_2$-error $\|{\cal B}_k(D)-\bar{z}_k(D_k)\|_2^2$ must be at least $\alpha_0^2$, with probability at least $\rho_0$ (both on $D_k$ and the internal randomness of ${\cal B}_k$). Letting $Y_k:=\mathbf{1}_{\{\|{\cal B}_k(D)-\bar{z}_k(D_k)\|_2\geq \alpha_0\}}$, we have
\begin{align*}
\mathbb{P}\Big[ \|{\cal A}(D)-\bar z(D)\|_2^2\geq \big(\frac{\alpha_0 n_0}{n}\big)^2 \frac{\rho_0 K}{2} \Big] 
&\geq  \mathbb{P}\Big[ \sum_{k=1}^K Y_k \geq \frac{\rho_0 K}{2} \Big]. 
\end{align*}
We will now use a coupling argument to lower bound the probability above. First, we let $U_1,\ldots,U_K\stackrel{i.i.d.}{\sim}\mbox{Unif}([0,1])$, and $W_k=\mathbf{1}_{\{U_i\geq \rho_0\}}$ which are i.i.d.. On the other hand, we define
\begin{align*}
p_k(y_1,\ldots,y_{k-1}) &:= \mathbb{P}[Y_k=1|Y_1=y_1,\ldots,Y_{k-1}=y_{k-1}]\\
\tilde Y_k &:= \mathbf{1}_{\{U_k\geq p_k(\tilde Y_1,\ldots,\tilde Y_{k-1})\}}.
\end{align*}
Noting that $Y\stackrel{d}{=}\tilde Y$, and that $\tilde Y_k\geq W_k$ almost surely, due to the fact that $p_k\geq \rho_0$ almost surely (which it follows from the $\ell_2$-error argument discussed above), we have
\[ \mathbb{P}\Big[ \sum_{k=1}^K Y_k \geq \frac{\rho_0 K}{2} \Big] =\mathbb{P}\Big[ \sum_{k=1}^K \tilde Y_k \geq \frac{\rho_0 K}{2} \Big] \geq \mathbb{P}\Big[ \sum_{k=1}^K W_k \geq \frac{\rho_0 K}{2} \Big]\geq 1-\exp(-\rho_0/8),  \]
where we used the one-sided multiplicative Chernoff bound.\\
Therefore, $\|{\cal A}(D)-\bar z(D)\|_2^2 \gtrsim \Big(\frac{\alpha_0n_0}{n}\Big)^2\rho_0 K$, with probability $1-\exp(-\rho_0/8)$.  We conclude that $\tilde \mu$ induces a $(\alpha,\rho)$-distributional lower bound for $(\varepsilon,\delta)$-DP mean estimation, as claimed.
\end{proof}

%% file: DP_SCO_Sparse_Gradients/Proofs/thm08LB_pure_DP_sparse.tex
\begin{proof}
By \Cref{lem:sparse_codes}, there exists a set ${\cal P}$ of $1/\sqrt{2}$-packing vectors on ${\cal C}_s^d$ with $\log(|{\cal P}|)\gtrsim s\log(d/s)$. \Cref{lem:packing_DP_LB} thus implies the desired lower bound.
\end{proof}

%% file: DP_SCO_Sparse_Gradients/Proofs/thm04LB_pure_DP_sparse_bootstrap.tex
With all the building blocks in place, we now prove \Cref{thm:LB_pure_DP_sparse_bootstrap}.

\begin{proof}[Proof of \Cref{thm:LB_pure_DP_sparse_bootstrap}]
We divide the analysis into the different regimes of sample size $n$. First, if $n\lesssim \frac{s\log(d/s)}{\eps}$, then \Cref{thm:LB_pure_DP_sparse} provides an $\Omega(1)$ lower bound.

Next we consider the case $ \frac{s\log(d/s)}{\eps}\lesssim n \lesssim \frac{d}{\eps}$. For $s\leq t\leq d$ to be determined, let $n_0=\frac{s\log(t/s)}{\eps}$. We choose $t$ so that $\frac{d}{t}\eqsim \frac{n}{n_0}$: this can be attained by choosing $t\eqsim \frac{ds}{\eps n}\log\big(\frac{d}{\eps n}\big)$. This implies in the context of \Cref{lem:block_diagonal_bootstrap} that $K=\frac{d}{t}\eqsim \frac{n}{n_0}$. By \Cref{thm:LB_pure_DP_sparse}, this implies a lower bound $\alpha_0\gtrsim 1$, with constant probability $1/2$ for sparse mean estimation in dimension $t$. By \Cref{lem:block_diagonal_bootstrap}, we conclude a sparse mean estimation lower bound of $ \frac{\alpha_0n_0}{n}\sqrt{\frac{K}{2}} \gtrsim \frac{1}{\sqrt{K}}\gtrsim \sqrt{\frac{s\log(d/n\varepsilon)}{\varepsilon n}}$ holds with constant probability.

On the other hand, if $n \gtrsim \frac{d}{\eps}$, let $n^*\eqsim\frac{d}{\varepsilon}$. By the previous paragraph, for datasets of size $n^*$ the following lower bound holds, $\Omega\Big(\sqrt{\frac{s\log(d/\varepsilon n ^*)}{\varepsilon n^*}}\Big)\gtrsim \sqrt{\frac{s}{d}}$. For any $n>n^*$, by \Cref{lem:bootstrap-privacy-lb}, we have the lower bound $\Omega\Big( \sqrt{\frac{s}{d}} \frac{n^*}{n}\Big) \gtrsim \frac{\sqrt{sd}}{\varepsilon n}$ holds with constant probability.
\end{proof}

%% file: DP_SCO_Sparse_Gradients/Proofs/thm06LB_approx_DP_sparse.tex
\begin{proof}
We divide the analysis into the different regimes of sample size $n$.  
First, if $n\lesssim \sqrt{s\ln(1/\delta)}/\eps$, then embedding an $s$-dimensional lower bound construction \citep{Bun:2017}\footnote{While \citep{Bun:2017} only provides 1-dimensional distributional lower bounds for approximate DP mean estimation, it is easy to convert these into higher dimensional lower bounds, see e.g.~\cite{Cai:2021,Kamath:2020}.} and padding it with zeros for the remaining $d-s$ features, provides an $\Omega(1)$ lower bound with constant probability.

Next, we consider the case $\sqrt{s\ln(1/\delta)}/\eps \lesssim n \lesssim \frac{d\sqrt{\ln(1/\delta)}}{\sqrt{s}\varepsilon}$. Let $n_0=\sqrt{s\ln(1/\delta)}/\eps$,  $t=s$, and $K=\frac{n}{n_0}\lesssim \frac{d}{s}$, where the last inequality holds by our regime assumption. The classic $s$-dimensional mean estimation lower bound by \citep{Bun:2017} provides a $\alpha_0\gtrsim 1$ lower bound with constant probability. Hence by \Cref{lem:block_diagonal_bootstrap}, the sparse mean estimation problem satisfies a lower bound $\Omega\big( \frac{\alpha_0 n_0}{n} \sqrt{K} \big)\gtrsim \frac{1}{\sqrt{K}}\gtrsim \frac{[s\ln(1/\delta)]^{1/4}}{\sqrt{\varepsilon n}}$, with constant probability.

We conclude with the final range, $ n \gtrsim \frac{d\sqrt{\ln(1/\delta)}}{\sqrt
{s}\eps}$. First, letting $n^{\ast}\eqsim \frac{d\sqrt{\ln(1/\delta)}}{\sqrt{s}\varepsilon}$, we note that this sample size falls within the range of the previous analysis, which implies a lower bound with constant probability of $\frac{[s\ln(1/\delta)]^{1/4}}{\varepsilon\sqrt{n^{\ast}}}\gtrsim \frac{\sqrt{s}}{\sqrt{d}}$.  
Now, if $n>n^{\ast}$, by \Cref{lem:bootstrap-privacy-lb}, we conclude that the following lower bound holds with constant probability, $ \Omega\Big( \frac{\sqrt{s}}{\sqrt{d}} \frac{n^{\ast}}{n} \Big)
\gtrsim \frac{\sqrt{d\ln(1/\delta)}}{n\eps}$.
\end{proof}

%% file: DP_SCO_Sparse_Gradients/s6-dp-erm-nonconvex.tex
\section{Bias-Reduction Method for DP-ERM with Sparse Gradients}

\label[section]{sec:bias_reduction}

We now start with our study of DP-ERM with sparse gradients.\ifthenelse{\boolean{journal}}{}{ We defer some proofs to \Cref{app:sec_bias_reduction_proofs}.} In this section and later, we will impose subsets of the following assumptions:
\begin{enumerate}[label=(A.\arabic*),itemsep=0pt]
\item {\sl Initial distance:} For SCO, $\|x^0-\optpop\|\leq D$; for ERM,  $\|x^0-\opterm\|\leq D$.\label{assumpt:init_dist}
\item {\sl Diameter bound:} $\|x-y\|\leq D$, for all $x,y\in \cX$.\label{assumpt:diameter}
\item {\sl Convexity:} $f(\cdot,z)$ is convex, for all $z\in \cZ$. \label{assumpt:cvx}
\item {\sl Loss range:} $f(x,z)-f(y,z)\leq B$, for all $x,y\in \cX$, $z\in \cZ$.\label{assumpt:loss_range} 
\item {\sl Lipschitzness:} $f(\cdot,z)$ is $\Lip$-Lipschitz, for all $z\in \cZ$. 
\label{assumpt:Lipschitz} 
\item {\sl Smoothness:} $\nabla f(\cdot,z)$ is $\smooth$-Lipschitz, for all $z\in \cZ$. \label{assumpt:smoothness} 
\item {\sl Individual gradient sparsity:} $\nabla f(x,z)$ is $s$-sparse, for all $x\in \cX$ and $z\in \cZ$. \label{assumpt:sparse_grad}  
\end{enumerate}

The most natural and popular DP optimization algorithms are based on SGD. Here we show how to integrate the mean estimation algorithms from \Cref{sec:MeanEstimationUB} to design a stochastic first-order oracle that can be readily used by any stochastic first-order method. The key challenge here is that estimators from \Cref{sec:MeanEstimationUB} are inherently biased, which is known to dramatically deteriorate the convergence rates.  
Hence, we start by introducing a bias-reduction method.

\subsection{Subsampled Bias-Reduced Gradient Estimator for DP-ERM}

We propose \Cref{alg:subsample_DB}, inspired by a debiasing technique proposed in \cite{Blanchet:2015}. 
The idea is the following: we know that the \ifthenelse{\boolean{journal}}{\Cref{alg:Gaussian_ell1_Rec}}{
projection mechanism\footnote{Note that we use the projection mechanism (\Cref{alg:proj}) as subroutine for \Cref{alg:subsample_DB} only to have a self-contained presentation in the main body of the paper. We will analyze and state the sharper bounds obtained with \Cref{alg:Gaussian_ell1_Rec} as subroutine.}} would provide more accurate gradient estimators with larger sample sizes, and we will see that its bias improves analogously. We choose our batch size as a random variable with exponentially increasing range, and given such a realization we subtract the projection mechanism applied to the whole batch minus the same mechanism applied to both halves of this batch.\footnote{We follow the Blanchet-Glynn notation of $O$ and $E$ to denote the `odd' and `even' terms for the batch partition \cite{Blanchet:2015}; this partitioning is arbitrary.} This subtraction, together with a multiplicative and additive correction, results in the expected value of the outcome ${\cal G}(x)$ corresponding to the estimator with the largest batch size, leading to its expected accuracy being boosted by such large sample size, without necessarily utilizing such amount of data (in fact, the probability of such batch size being picked is polynomially smaller, compared to the smallest possible one). The caveat with this technique, as we will see, relates to a heavy-tailed distribution of outcomes, and therefore great care is needed for its analysis.

\ifthenelse{\boolean{journal}}{
Instrumental to our analysis is the following {\em truncated geometric distribution}.
\begin{definition}
A random variable $N$ has a truncated geometric distribution with parameter $M\in \mathbb{N}$ iff its probability mass is supported on $\{0,\ldots,M\}$, with corresponding probabilities
    \begin{equation} \label{eqn:probas_finite_batch_size}
    p_k =\left\{
    \begin{array}{cc}
    C_M/2^{k} & k\leq M \\
    0 & k>M,
    \end{array}
    \right.
    \end{equation}
where $C_M=[2(1-2^{-(M+1)})]^{-1}$ is the normalizing constant. Its law will be denoted by $\TG(M)$.
\end{definition}

\begin{remark}
     Note that $1/2\leq C_M \leq 1$. I.e., the normalizing constant of the truncated geometric distribution is uniformly bounded away from $0$ and $+\infty$.
\end{remark}
}{
Instrumental to our analysis is the following {\em truncated geometric distribution} with parameter $M\in\mathbb{N}$, whose law will be denoted by $\TG(M)$: we say $N\sim \TG(M)$ if it is supported on $\{0,\ldots,M\}$, and takes value $k$ with probability $p_k:=C_M/2^k$, where 
$C_M=(2(1-2^{-(M+1)}))^{-1}$, is the normalizing constant. 
Note that $1/2\leq C_M \leq 1$, thus it is bounded away from $0$ and $+\infty$.
}

\begin{algorithm}[t]
\cprotect\caption{%
\texttt{Subsampled\_Bias-Reduced\_Gradient\_Estimator}%
$(x,S, N,\varepsilon, \delta)$} \label{alg:subsample_DB}
\begin{algorithmic}
\Require Dataset $S=(z_1,\ldots,z_n)\in {\cal Z}^n$, $\varepsilon,\delta>0$ privacy parameters, $\Lip$-Lipschitz loss $f(x,z)$ with $s$-sparse gradient, $x\in {\cal X}$, batch size parameter $N\sim \TG(M)$ with $M=\lfloor \log_2(n) \rfloor-1$ 
\State Let $B\sim \Unif\big(\binom{n}{2^{N+1}}\big)$, $O,E$ a partition of $B$ with $|O|=|E|=2^N$, $I\sim \Unif([n])$ 
\ifthenelse{\boolean{journal}}{
\State $G_{N+1}^+(x,B)=\texttt{Gaussian $\ell_1$-Recovery}(\nabla F_{B}(x),\varepsilon/4,\delta/4,2^{N+1})$ (\Cref{alg:Gaussian_ell1_Rec}) 
\State $G_{N}^-(x,O)=\texttt{Gaussian $\ell_1$-Recovery}(\nabla F_{O}(x),\varepsilon/4,\delta/4,2^N)$
\State $G_{N}^-(x,E)=\texttt{Gaussian $\ell_1$-Recovery}(\nabla F_{E}(x),\varepsilon/4,\delta/4,2^N)$
\State $G_0(x,I)=\texttt{Gaussian $\ell_1$-Recovery}(\nabla f(x,z_I),\varepsilon/4,\delta/4,1)$
}{
\State $G_{N+1}^+(x,B)=\text{\texttt{Projection\_Mechanism}}(\nabla F_{B}(x),\varepsilon/4,\delta/4,2^{N+1})$ (\Cref{alg:proj}) 
\State $G_{N}^-(x,O)=\text{\texttt{Projection\_Mechanism}}(\nabla F_{O}(x),\varepsilon/4,\delta/4,2^N)$
\State $G_{N}^-(x,E)=\text{\texttt{Projection\_Mechanism}}(\nabla F_{E}(x),\varepsilon/4,\delta/4,2^N)$
\State $G_0(x,I)=\text{\texttt{Projection\_Mechanism}}(\nabla f(x,z_I),\varepsilon/4,\delta/4,1)$
}
\State {\bf Return} (below $p_k=\mathbb{P}[\TG(M)=k]$)
\[\textstyle
{\cal G}(x) = \frac{1}{p_{N}}\Big( G_{N+1}^+(x,B)-\frac12 \big(G_{N}^-(x,O)+G_{N}^-(x,E) \big) \Big) + G_0(x,I)
\] 
\end{algorithmic}
\end{algorithm}

\begin{algorithm}[t]
\cprotect\caption{%
\texttt{Subsampled\_Bias-Reduced\_Sparse\_SGD}$(x^0,S,\varepsilon, \delta)$}
\label{alg:subsampled_DB_SGD}
\begin{algorithmic}
\Require Initialization $x^0\in{\cal X}$; Dataset $S=(z_1,\ldots,z_n)\in\cZ^n$; $\varepsilon,\delta$, privacy parameters; stepsize $\eta>0$; gradient oracle for $\Lip$-Lipschitz and with $s$-sparse gradient loss $f(\cdot,z)$
\State $t\gets -1$ 
\While{$\sqrt{2\ln\big(\frac{4}{\delta}\big)\sum_{s=0}^{t-1}\big(\frac{3\cdot2^{N_s+1}+1}{16n}\big)^2}+\frac{\varepsilon}{2}\sum_{s=0}^{t-1}\big(\frac{3\cdot2^{N_s+1}+1}{16n}\big)^2 \leq \frac{1}{2}$ and $\sum_{s=0}^{t-1}\frac{3\cdot 2^{N_s+1}+1}{16n}\leq \frac14$}
    \State $t\gets t+1$
    \State $N_{t}\sim \TG(M)$ where $M=\lfloor\log_2(n)\rfloor-1$
    \State ${\cal G}(x^t)=$ \text{\texttt{\texttt{Subsampled\_Bias-Reduced\_Gradient\_Estimator}}}$(x^t,S,N_t,\varepsilon/8, \delta/4)$ (Alg.~\ref{alg:subsample_DB})    
        \State $x^{t+1} = \Pi_{\cal X}\big[x^t-\eta {\cal G}(x^t)\big]$   
\EndWhile
\State \Return $\begin{cases}
\bar{x}=\frac{1}{t+1}\sum_{s=0}^t x^s &\text{if }f(\cdot,z)\text{ is convex}\,,\\[1mm]
x^{\hat t} \text{ where } 
\hat{t}\sim\Unif(\{0,\ldots,T\})& \,\text{ if}f(\cdot,z)\text{ is not convex}.\\
\end{cases}$
\end{algorithmic}
\end{algorithm}

We propose \Cref{alg:subsampled_DB_SGD}, which interacts with the oracle given in \Cref{alg:subsample_DB}. For convenience, we will denote the random realization from the truncated geometric distribution used at iteration $t$ by $N_t$. The idea is that, using the fully adaptive composition property of DP~\citep{Whitehouse:2023}, we can run the method until our privacy budget is exhausted. Due to technical reasons, related to the bias-reduction, we need to shift by one the termination condition in the algorithm. In particular, our algorithm goes over the reduced privacy budget of  $(\varepsilon/2,\delta/2)$. The additional slack in the privacy budget guarantees that even with the extra oracle call the algorithm respects the privacy constraint.

\begin{lemma} \label[lemma]{lem:privacy_subsampled_DB}
Algorithm \ref{alg:subsampled_DB_SGD} is $(\varepsilon,\delta)$-DP.
\end{lemma}

\ifthenelse{\boolean{journal}}{\input{DP_SCO_Sparse_Gradients/Proofs/lem16privacy_subsampled_DB}}{}

\subsection{Bias and Moment Estimates for the Debiased Gradient Estimator}

We provide  bias and second moment estimates for our debiased estimator of the empirical gradient. In summary, we show that this estimator has bias matching that of the full-batch gradient estimator, while at the same time its second moment is bounded by a mild function of the problem parameters. 

\begin{lemma} \label{lem:bias_variance_subsampled_DB}
Let $d \gtrsim \frac{n\varepsilon\sqrt{s\ln(d/s)}}{\sqrt{\ln(1/\delta)}}$. 
Algorithm \ref{alg:subsample_DB}, enjoys bias and second moment bounds
    \begin{align*}  
    \Big\|\mathbb{E}[{\cal G}(x)-\nabla F_S(x)|x] \Big\| &\textstyle \lesssim~ \frac{ \Lip[s\ln(d/s)\ln(1/\delta)]^{1/4}}{\sqrt{n\varepsilon}}=:b, \\
    \mathbb{E}[\|{\cal G}(x)\|^2|x] &\textstyle \lesssim~  \frac{\Lip^2\ln(n)\sqrt{s\ln(d/s)\ln(1/\delta)}}{\varepsilon}=:\nu^2. 
    \end{align*}
\end{lemma}

\input{DP_SCO_Sparse_Gradients/Proofs/lem15bias_variance_subsampled_DB}

\subsection{Accuracy Guarantees for Subsampled Bias-Reduced Sparse SGD}

The previous results provide useful information about the privacy, bias, and second-moment of our proposed oracle. Our goal now is to provide excess risk rates for DP-ERM. For this, we need to prove the algorithm runs for long enough, i.e.,~a lower bound on the stopping time 
of Algorithm \ref{alg:subsampled_DB_SGD}, 
\begin{equation} \label{eqn:stop_subsample_privacy_budget}
\textstyle
T:=\inf\Big\{ t: \frac{\varepsilon}{2} < \varepsilon\Big(2\ln\big(\frac{4}{\delta}\big){\displaystyle\sum_{s=0}^{t}}\big(\frac{3\cdot2^{N_s+1}+1}{16n}\big)^2\Big)^{1/2}  +\frac{\varepsilon^2}{2}{\displaystyle\sum_{s=0}^{t}}\frac{3\cdot2^{N_s+1}+1}{16n}  \,\,\mbox{ or }\,\, \frac{\delta}{4}<{\displaystyle\sum_{s=0}^{t}} \frac{(3\cdot2^{N_s+1}+1)\delta}{16n} \Big\}. 
\end{equation}

\ifthenelse{\boolean{journal}}{
\subsubsection{Lower Bounds on the Stopping Time}
}{}
The proof of \Cref{lem:bias_variance_subsampled_DB} implies that moments of ${\cal G}$ increase exponentially in $M$. This heavy-tailed behavior implies that $T$ may not concentrate strongly enough to obtain high-probability lower bounds for $T$. What we will do instead is showing that {\em with constant probability} $T$ behaves as desired.

To justify the approach, let us provide a simple in-expectation bound on how the privacy budget accumulates in the definition of $T$: letting $\varepsilon_t=(3\cdot2^{N_t+1}+1)\varepsilon/[16n]$, we have that 
\[ \ifthenelse{\boolean{journal}}{}{\textstyle}
\mathbb{E}\Big[ \sum_{s=0}^t \varepsilon_s^2\Big] 
= \frac{(t+1)\varepsilon^2}{(16n)^2} \mathbb{E}\Big[(3\cdot 2^{N_1+1}+1)^2\Big] 
\leq \frac{2(t+1)\varepsilon^2}{(16n)^2} \Big(9 \mathbb{E}[2^{2(N_1+1)}]+1\big] \Big) \lesssim \frac{t\varepsilon^2}{n}, 
\]
where in the last step we used that 
$\mathbb{E}\big[2^{2(N_1+1)}\big] = C_M\sum_{k=1}^{M+1} 2^k 
\lesssim n$. This in-expectation analysis can be used in combination with ideas from stopping times to establish  bounds for $T$.
\begin{lemma} \label{lem:Run_time_Markov}
Let $0<\delta<1/n^2$. 
\ifthenelse{\boolean{journal}}{In the notation of  Algorithm \ref{alg:subsampled_DB_SGD}, let}{Let} 
$T$ be the stopping time defined in eqn.~\eqref{eqn:stop_subsample_privacy_budget}. Then, there exists  $t=Cn/\log(2/\delta)$ (with $C>0$ an absolute constant) such that
\ifthenelse{\boolean{journal}}{
\[\mathbb{P}[T\leq t] \leq 1/4. \]
}{$\mathbb{P}[T \leq t] \leq 1/4.$}
On the other hand, 
\[ \ifthenelse{\boolean{journal}}{}{\textstyle}
\frac{n^2}{(n+1)\ln(4/\delta)}-1 \leq \mathbb{E}[T] \leq \frac{64n}{9\ln(4/\delta)}. \]
\end{lemma}

\ifthenelse{\boolean{journal}}{
\input{DP_SCO_Sparse_Gradients/Proofs/lem15Run_time_Markov}

}{}

\ifthenelse{\boolean{journal}}{
    \subsubsection{Excess Empirical Risk in the Convex Setting}

    \label{sec:DP_ERM_bias_reduction}

\input{DP_SCO_Sparse_Gradients/s6.3.1convex_ERM_bias_reduction}

    \input{DP_SCO_Sparse_Gradients/s6.3.2nonconvex_ERM_bias_reduction}

}{
With our bounds on $T$, further analysis involving regret bounds on randomly stopped SGD yields the following bounds for convex and nonconvex losses. See \Cref{thm:convex_const_proba} and \Cref{thm:nonconvex_const_proba} for details.

\begin{theorem} \label{thm:biased_SGD_merged}
Consider a \eqref{eqn:SO} problem under initial distance (\Cref{assumpt:init_dist}), Lipschitzness (\Cref{assumpt:Lipschitz}) and gradient sparsity (\Cref{assumpt:sparse_grad}) assumptions. 
\begin{itemize}
\item In the convex case (\Cref{assumpt:cvx}), \Cref{alg:subsampled_DB_SGD} satisfies
\[ \mathbb{P}\Big[ F_S(\hat x)-F_S(\opterm) \lesssim \Lip D\, \frac{\sqrt{\ln n}[s\ln(d/s)\ln^3(1/\delta)]^{1/4}}{\sqrt{\varepsilon n}} \Big] \geq \frac12.
\]
\item In the nonconvex case, additionally assuming
smoothness (\Cref{assumpt:smoothness}) and the following {\em initial suboptimality assumption}: namely, that given our initialization $x^0\in \mathbb{R}^d$, there exists $\gap>0$ such that $F_S(x^0)-F_S(\opterm)\leq \gap$; \Cref{alg:subsampled_DB_SGD} satisfies 
\[\mathbb{P}\Big[ \|\nabla F_S(x^{\hat t})\|_2^2 \lesssim \big(\sqrt{\gap\smooth}\Lip\sqrt{\ln(n)\ln(1/\delta)}+\Lip^2\big) \frac{[s\ln(d/s)\ln(1/\delta)]^{1/4}}{\sqrt{\varepsilon n}} \Big] \geq \frac12. \]
\end{itemize}
\end{theorem}
}

\ifthenelse{\boolean{journal}}{\input{DP_SCO_Sparse_Gradients/Proofs/thm16nonconvex_const_proba}}{
\paragraph{Boosting the Confidence of the Bias-Reduced SGD} 
To conclude, in \Cref{sec:boosting} we provide a boosting algorithm that can exponentially amplify the success probability of \Cref{alg:subsampled_DB_SGD}. The approach is based on  making parallel runs of the method and using private model selection to obtain the best performing model.
}

%% file: DP_SCO_Sparse_Gradients/Proofs/lem16privacy_subsampled_DB.tex
\begin{proof}
The proof is based on the fully adaptive composition theorem of differential privacy \citep{Whitehouse:2023}. For this, we consider $\{{\cal A}_t\}_{t\geq 0}$, where ${\cal A}_0(S)=(x^0,N_0)$ (here $N_0$ the first truncated geometric parameter), and inductively, ${\cal A}_{t+1}({\cal A}_{t}(S),S)$ for $t\geq 0$ takes as input ${\cal A}_{t}(S)=(x^{t},N_{t})$, computes ${\cal G}(x_t)$ using the subsampled debiased gradient estimator (Algorithm \ref{alg:subsample_DB}), and performs a projected gradient step based on ${\cal G}(x^{t})$. Let ${\cal H}_t$ be the $\sigma$-algebra induced by $({\cal A}_s)_{s=0,\ldots,t}$.

Suppose now that ${\cal A}_t$ is $(\varepsilon_t,\delta_t)$-DP, where $(\varepsilon_t,\delta_t)$ are ${\cal H}_t$-measurable  (we will later obtain these parameters), and let $T:=\inf\{t: \varepsilon_{[0:t]}>\varepsilon/2,\,\,\delta_{[0:t]}>\delta/2\}$, in the language of \Cref{thm:DP_Filter} (notice that in the context of that theorem, we are choosing $\delta^{\prime}=\delta^{\prime\prime}=\delta/4$). 
We first claim that $(x^t)_{t=0,\ldots,T-1}$ is $(\varepsilon/2,\delta/2)$-DP, which follows directly from \Cref{thm:DP_Filter}. 
Next, we will later show that $\varepsilon_t\leq \varepsilon/4$ and  $\delta_t\leq \delta/4$, almost surely (this applies in particular to  $x_T$), and therefore by the composition property of DP, $(x_t)_{t\leq T}$ is $(\varepsilon,\delta)$-DP.

Next, we provide the bounds on $(\varepsilon_t,\delta_t)$ required to conclude the proof.
For this, we first note that --conditionally on $x^t$, $N_t$ and $B_t$-- the computation of $G_{N_t+1}^+(x^t,B_t)$, $G_{N_t}^-(x^t,O_t)$, $G_{N_t}^-(x^t,E_t)$, is $(3\varepsilon/32,3\delta/16)$-DP. Furthermore, by privacy amplification by subsampling, this triplet of random variables is $(\varepsilon^{\prime},\delta^{\prime})$, with 
\[ 
\varepsilon^{\prime} =\ln\Big( 1+\frac{2^{N_t+1}}{n}(e^{3\varepsilon/32}-1)\Big) \leq \frac{2^{N_t+1}}{n}\frac{3\eps}{16},\qquad 
\delta^{\prime} = \frac{2^{N_t+1}}{n}\frac{3\delta}{16},
\]
where we used above that $\varepsilon\leq 1$. Similarly, we have that $G_0(x,I)$ is $\big(\frac{\varepsilon}{16n},\frac{\delta}{16n}\big)$-DP. Therefore, by the basic composition theorem of DP, we have the following privacy parameters for the $t$-th iteration of the algorithm
\[ \varepsilon_t=(3\cdot2^{N_t+1}+1)\frac{\varepsilon}{16n},\quad \delta_t=(3\cdot 2^{N_t+1}+1)\frac{\delta}{16n}. \]
This proves in particular that $(\varepsilon_t,\delta_t)$ are ${\cal H}_t$-measurable, and that $\varepsilon_t\leq\varepsilon/4$, and $\delta_t\leq \delta/4$ almost surely, which concludes the proof
\end{proof}

%% file: DP_SCO_Sparse_Gradients/Proofs/lem15bias_variance_subsampled_DB.tex
\begin{proof}
    For simplicity, we assume without loss of generality that $n$ is a power of 2, so that $2^{M+1}=n$.

\noindent{\bf Bias } 
Let, for $k=0,\ldots,M$, 
$G_{k+1}^{+}(x)=\mathbb{E}[G_{N+1}^+(x,B) \mid N=k,x]$, and 
\[G_k^{-}(x)=\mathbb{E}[G_N^-(x,E) \mid N=k,x]=\mathbb{E}[G_N^-(x,O) \mid N=k,x],\] where the last equality follows from the identical distribution of $O$ and $E$. Noting further that $G_k^+(x)=G_{k}^-(x)$ (which follows from the uniform sampling and the cardinality of the used datapoints), and using the law of total probability,
we have
    \begin{eqnarray*}
        \mathbb{E}[{\cal G}(x) \mid x] 
        &=&\textstyle \sum_{k=0}^{M} \Big( G_{k}^+(x)-G_{k-1}^-(x) \Big)+\mathbb{E}[G_0(x,I) \mid x] \\
        &=& \textstyle G_{M+1}^+(x)-G_0^-(x)+\mathbb{E}[G_0(x,I) \mid x] \\
        &=& \textstyle \mathbb{E}[G_{M+1}^+(x)-\nabla F_S(x)|x] + \nabla F_S(x),
    \end{eqnarray*}
    where we also used that $\mathbb{E}[G_0(x,I) \mid x]=G_0^-(x)$ (since $I$ is a singleton). 
    Next, by \Cref{thm:Compressed_Sensing_UB_sparse_Mean_Est} 
    \[\textstyle
    \| \mathbb{E}[{\cal G}(x) \mid x]-\nabla F_S(x)\| \leq  \|\mathbb{E}[G_{M+1}^+(x)-\nabla F_S(x)|x]\| 
    \lesssim \Lip\frac{[s\ln(d/s)\ln(1/\delta)]^{1/4}}{\sqrt{n\varepsilon}}.
    \]

\noindent{\bf Second moment bound } 
Using the law of total probability, and that $O,E$ are a partition of $B$:
    \begin{align*}
        &\mathbb{E}[\|{\cal G}(x)\|^2 \mid x] 
        = \sum_{k=0}^M p_k \mathbb{E}\Big[ \Big\|\frac{1}{p_k}[G_{N+1}^+(x,B)-\nabla F_B(x)]\\
        &-\frac{1}{2p_k}\big[G_N^-(x,O)-\nabla F_O(x)+G_N^-(x,E)-\nabla F_E(x)\big]+G_0(x,I)\Big\|^2 \Big| x, N=k\Big] \\
        &~\leq~2\mathbb{E}[\|G_0(x,I)\|^2 \mid x] + 4 \sum_{k=0}^M \frac{1}{p_k} \mathbb{E}\Big[ \Big\|G_{N+1}^+(x,B)-\nabla F_B(x)\Big\|^2\Big|x,N=k\Big]\\
        & ~~~+~\sum_{k=0}^M\frac{1}{p_k}\mathbb{E}\Big[\Big\| G_N^-(x,O)-\nabla F_O(x)\Big\|^2+\Big\| G_N^-(x,E)-\nabla F_E(x)\Big\|^2\Big|x,N=k\Big].
    \end{align*}
    We now use \Cref{thm:Compressed_Sensing_UB_sparse_Mean_Est}, to conclude that
    \begin{align*} 
    \textstyle\mathbb{E}\Big[ \Big\|G_{N+1}^+(x,B)-\nabla F_B(x)\Big\|^2 ~\Big|~ x,N=k\Big]
    &\textstyle\lesssim \frac{\Lip^2 \sqrt{s\ln(d/s) \ln(1/\delta)}}{2^{k+1}\eps } \\
    \textstyle\max_{A\in\{O,E\}}\Big\{ \mathbb{E}\Big[\Big\| G_N^-(x,A)-\nabla F_A(x)\Big\|^2 ~\Big|~  x,N=k\Big]  \Big\}
    &\textstyle\lesssim \frac{\Lip^2 \sqrt{s\ln(d/s) \ln(1/\delta)}}{2^{k}\eps } \\
    \textstyle\mathbb{E}\big[\big\| G_0(x,I)\big\|^2 ~\big|~ x\big]
    &\textstyle\lesssim \frac{\Lip^2 \sqrt{s\ln(d/s) \ln(1/\delta)}}{\eps }.   
    \end{align*}
    Recalling that $M+1=\log_2 n$ and $p_k=2^{-k}$, these bounds readily imply that $\mathbb{E}\|{\cal G}(x)\|^2\lesssim\nu^2$.
\end{proof}

%% file: DP_SCO_Sparse_Gradients/Proofs/lem15Run_time_Markov.tex
\begin{proof}   
Let $A=\sum_{s=0}^{t-1}\big(\frac{3\cdot 2^{N_s+1}+1}{16n}\big)^2$, and note that for $t\leq T+1$, $A\leq 1$ almost surely. Then, we have that 
\[\varepsilon_{[0:t-1]}=\sqrt{2\ln(4/\delta)\varepsilon^2 A}+\frac{\varepsilon^2}{2}A \leq 2\varepsilon\sqrt{2\ln(4/\delta) A}.
\]
Now, by eqn.~\eqref{eqn:stop_subsample_privacy_budget} and the union bound,
\begin{align*}
    \mathbb{P}[T\leq t] &\leq \mathbb{P}\Big[ 2\varepsilon\sqrt{2\ln(4/\delta)A}>\varepsilon/2\Big]+\mathbb{P}\Big[\sum_{s=0}^{t-1}(3\cdot2^{N_t+1}+1) >4n \Big]\\
    &\leq \mathbb{P}\Big[ \sum_{s=0}^{t-1}\big(3\cdot 2^{N_t+1}+1\big)^2> \frac{32n^2}{\ln\big(\frac{4}{\delta}\big)}\Big]+\mathbb{P}\Big[\sum_{s=0}^{t-1}(3\cdot2^{N_t+1}+1) >4n \Big]\\
    &\leq \frac{t\ln\big(\frac{4}{\delta}\big)}{16n^2}\big(9\mathbb{E}[2^{2 (N_t+1)}]+1\big)+\frac{t}{4n}[6(M+1)+1]\\
    &\leq  \frac{t\ln\big(\frac{4}{\delta}\big)}{16n^2}[18n+1]+\frac{t}{4n}[6\log(n)+1]\\
    &\leq 1/4,
    \end{align*}
    where the third step follows from Markov's inequality and the fact that $(N_s)_s$ are i.i.d., and the last step follows from our choice of $t=Cn/\log(4/\delta)$ with $C>0$ sufficiently small (here we use the fact that $\delta<1/n^2$).

    For the second part, we use that by the definition of $T$ (eqn.~\eqref{eqn:stop_subsample_privacy_budget})
    \begin{align*} 
    &\frac{\varepsilon}{2} < \sqrt{ 2\varepsilon^2\ln\big(\frac{4}{\delta}\big)\sum_{s=0}^{T}\frac{(3\cdot 2^{N_s+1}+1)^2}{(16n)^2} } +\frac{\varepsilon^2}{2} \sum_{s=0}^{T}\frac{(3\cdot 2^{N_s+1}+1)^2}{(16n)^2} \,\,\, \vee\,\,\, \frac{1}{4}<\sum_{s=0}^{T}\frac{3\cdot 2^{N_s+1}+1}{16 n}\\
    & \Longrightarrow\quad 
     n^2 < \max\left\{ 8\ln\big(\frac{4}{\delta}\big)\sum_{s=0}^{T}\frac{(3\cdot 2^{N_s+1}+1)^2}{(16)^2} , n\sum_{s=0}^{T}\frac{3\cdot 2^{N_s+1}+1}{4}
     \right\}
    \end{align*}
    Taking expectations and bounding the maximum by the sum allows us to use Wald's identity as follows,
    \begin{align*} 
    n^2 &< \mathbb{E}[T+1] \Big(8\ln\big(\frac{4}{\delta}\Big)\frac{2(9n+1)}{16^2}+n\frac{3\log(n)+1}{4}\Big)\\
    &\leq \mathbb{E}[T+1] \ln\big(\frac{4}{\delta}\Big)(n+1),
    \end{align*}
    which proves the claimed bound.

    The upper nound on $\mathbb{E}[T]$ is obtained similarly. Again, by eqn.~\eqref{eqn:stop_subsample_privacy_budget}, 
    \begin{align*}
    \frac{32n^2}{\ln(4/\delta)} \geq \mathbb{E}\Big[ \sum_{s=0}^{T-1} \big(3\cdot2^{N_s+1}+1)\big)^2 \Big] \geq \mathbb{E}[T] \frac{9n}{2}.
    \end{align*}
    Re-arranging terms provides the claimed lower bound.
\end{proof}

%% file: DP_SCO_Sparse_Gradients/s6.3.1convex_ERM_bias_reduction.tex
As a first application, we study the accuracy guarantees of Algorithm \ref{alg:subsampled_DB_SGD} in the convex setting. We remark that these rates will be slightly weaker than those provided in \Cref{sec:Output_Perturbation}, but this example is useful to illustrate the technique. 
Towards this goal,
we analyze the cumulative regret of the algorithm, namely ${\cal R}_T:=\sum_{t=0}^T [F_S(x^t)-F_S(\opterm)]$. Although this is a standard and well-studied object in optimization, we need to obtain bounds for this object when the stopping time $T$ is random. The key observation here is that since $T$ is a stopping time, the event $\{T\geq t\}$ is ${\cal F}_{t-1}$-measurable (here and throughout, ${\cal F}_t=\sigma((x_s)_{s\leq t})$ is the natural filtration). This permits using our bias and second moment bounds similarly to the case where $T$ is deterministic.\footnote{This idea is related to the Wald identities \cite{Rosenthal:2006}; however, we provide a direct analysis for the sake of clarity.} Moreover, 
for the sake of analysis, we will consider Algorithm \ref{alg:subsampled_DB_SGD} as running indefinitely, for all $t\geq 0$. This would of course eventually violate privacy. However, since our algorithm stops at time $T$, then privacy is guaranteed as done earlier in this section.

\begin{proposition} \label{prop:stopping_time_regret}
Let ${\cal R}_t:=\sum_{t=0}^t[F_S(x^t)-F_S(\opterm)]$, let $T$ be the stopping time defined in eqn.~\eqref{eqn:stop_subsample_privacy_budget}.  
Then
\[ \mathbb{E}[{\cal R}_T] \leq \frac{1}{2\eta}\|x^0-\opterm\|^2+\mathbb{E}[T+1]\big(\frac{\eta\nu^2}{2}+Db\big), \]
where $b$ and $\nu^2$ are defined as in Lemma \ref{lem:bias_variance_subsampled_DB}.
\end{proposition}

\begin{proof}
By Proposition \ref{prop:random_stop_SGD} (see \Cref{app:biased_SGD}),
\begin{align*}
&\mathbb{E}[{\cal R}_T]\\ 
&\leq \mathbb{E}\Big(\frac{1}{2\eta}\|x^0-\opterm\|^2+\sum_{t=0}^{T}\big[ \frac{\eta}{2}\|{\cal G}(x^t)\|^2+\langle \nabla F(x^t)-{\cal G}(x^t),x^t-\opterm\rangle \big] \Big) \\
&= \mathbb{E}\Big(\frac{1}{2\eta}\|x^0-\opterm\|^2\\
&+\sum_{t=0}^{\infty}\Big\{ \frac{\eta}{2}\mathbb{E}[\mathbf{1}_{\{T\geq t\}}\|{\cal G}(x^t)\|^2|{\cal F}_{t-1}]+\mathbb{E}[\mathbf{1}_{\{T\geq t\}}\langle \nabla F(x^t)-{\cal G}(x^t),x^t-\opterm\rangle|{\cal F}_{t-1}]\Big\} \Big) \\
&=\mathbb{E}\Big(\frac{1}{2\eta}\|x^0-\opterm\|^2\\
&+\sum_{t=0}^{\infty}\Big\{ \frac{\eta\mathbf{1}_{\{T\geq t\}}}{2}\mathbb{E}[\|{\cal G}(x^t)\|^2|{\cal F}_{t-1}]+\mathbf{1}_{\{T\geq t\}}\mathbb{E}[\langle \nabla F(x^t)-{\cal G}(x^t),x^t-\opterm\rangle|{\cal F}_{t-1}] \Big\} \Big) 
\end{align*}
where in the first equality we used the tower property of the conditional expectation, and in the second equality we used that $\{T\geq t\}=\{T\leq t-1\}^c$ is ${\cal F}_{t-1}$-measurable.

Now, by Lemma \ref{lem:bias_variance_subsampled_DB}, $\mathbb{E}[\langle \nabla F(x^t)-{\cal G}(x^t),x^t-\opterm\rangle|{\cal F}_{t-1}]\leq Db$ and $\mathbb{E}[\|{\cal G}(x^t)\|^2|{\cal F}_{t-1}]\leq \nu^2$ (note that ${\cal F}_{t-1}$ does not include the randomness of $N_t$, and therefore the bias and moment estimates as in the mentioned lemma hold), thus
\[ \mathbb{E}[{\cal R}_T] \leq \frac{1}{2\eta}\|x^0-\opterm\|^2+\mathbb{E}[T+1]\big(\frac{\eta\nu^2}{2}+Db\big). \]
\end{proof}

We conclude with the constant probability guarantee for the biased and randomly stopped SGD, \Cref{alg:subsampled_DB_SGD}.

\begin{theorem} \label{thm:convex_const_proba}
Consider a \eqref{eqn:SO} problem under convexity (\Cref{assumpt:cvx}), initial distance (\Cref{assumpt:init_dist}), Lipschitzness (\Cref{assumpt:Lipschitz}) and gradient sparsity (\Cref{assumpt:sparse_grad}) assumptions. 
Let $\tau=\frac{C^{\prime}n}{\ln(2/\delta)}$, where $C^{\prime}>0$ is an absolute constant. Let  
$\eta=\frac{D}{\nu\sqrt{\tau}}$, $U=CD[\nu\sqrt{\tau}+b\tau]$, where $C>0$ is an absolute constant. 
Then \Cref{alg:subsampled_DB_SGD} satisfies
\[ \mathbb{P}\Big[F_S(\bar x)-F_S(\opterm)\leq \frac{U}{\tau} \Big] \geq 1/2. \]
\end{theorem}

\begin{proof}
We start by noting that 
\begin{align*}
\mathbb{P}\Big[F_S(\bar x)-F_S(\opterm) > \frac{U}{\tau}\Big]
&\leq \mathbb{P}\big[\{T\leq \tau\} \cup \{{\cal R}_T> U\}\big] \leq \mathbb{P}\big[T\leq \tau] + \mathbb{P}[{\cal R}_T> U].
\end{align*}
For the first event, by Lemma \ref{lem:Run_time_Markov}, we have that $\mathbb{P}[T\leq \tau]\leq 1/4$ (which determines $C^{\prime}$). On the other hand, using Proposition \ref{prop:stopping_time_regret} and Lemma \ref{lem:Run_time_Markov}, we have that for our choice of $\eta$, we have that
    \[ \mathbb{E}[{\cal R}_T] \leq \frac{D\nu \sqrt{\tau}}{2}+\mathbb{E}[T+1]D\Big(\frac{\nu}{2\sqrt{\tau}}+ b\Big) \lesssim D[\nu \sqrt{\tau}+ \tau b].\]
    In particular, for our choice of $U$ (with $C>0$ sufficiently large),
    \[ \mathbb{P}[{\cal R}_T>U] \leq \frac{\mathbb{E}[{\cal R}_{T} ]}{U}\leq \frac14.\]
\end{proof}

The above result implies a nearly optimal empirical excess risk rate for DP-SCO,
\[ O\Big(\Lip D\, \frac{\sqrt{\ln n}[s\ln(d/s)\ln^3(1/\delta)]^{1/4}}{\sqrt{\varepsilon n}} \Big),\]
but only with constant probability. We defer to the next section how to boost this guarantee to hold with arbitrarily high probability.

%% file: DP_SCO_Sparse_Gradients/s6.3.2nonconvex_ERM_bias_reduction.tex
\subsubsection{Near Stationary Points for the Empirical Risk}

    For nonconvex objectives it is known that obtaining vanishing excess risk is computationally difficult. Hence, we study the more modest goal of approximating stationary points, i.e.~points with small norm of the gradient.  
    By combining known analyses of biased SGD with our bias-reduced oracle, we can establish bounds on the success probability of the algorithm.

     \begin{theorem} \label{thm:nonconvex_const_proba}
    Consider a (nonconvex) \eqref{eqn:SO} problem, under the following assumptions: Lipschitzness (\Cref{assumpt:Lipschitz}), smoothness (\Cref{assumpt:smoothness}), gradient sparsity (\Cref{assumpt:sparse_grad}), and the following {\em initial suboptimality assumption}: namely, that given our initialization $x^0\in \mathbb{R}^d$, we know $\gap>0$ such that
    \begin{equation} \label{eqn:initial_subopt}
    F_S(x^0)-F_S(\opterm)\leq \gap.
    \end{equation}
    Let $\tau=\frac{C^{\prime}n}{\ln(2/\delta)}$ with $C^{\prime}>0$ an absolute constant. Let $\eta=\sqrt{\frac{\gap}{\smooth t \nu^2}}$ and 
    $U=C\big(\sqrt{\gap\smooth\tau}\nu+\Lip \tau b\big)$ with $C>0$ an  absolute constant. Then \Cref{alg:subsampled_DB_SGD} satisfies  
    $\mathbb{P}\Big[ \|\nabla F_S(x^{\hat t})\|_2^2 \leq \frac{U}{\tau} \Big] \geq 1/2$, and
    \[ 
    \frac{U}{\tau} \lesssim \big(\sqrt{\gap\smooth}\Lip\sqrt{\ln(n)\ln(1/\delta)}+\Lip^2\big) \frac{[s\ln(d/s)\ln(1/\delta)]^{1/4}}{\sqrt{\varepsilon n}}. 
    \]
    \end{theorem}

    \input{DP_SCO_Sparse_Gradients/Proofs/thm16nonconvex_const_proba}

%% file: DP_SCO_Sparse_Gradients/Proofs/thm16nonconvex_const_proba.tex
\begin{proof} 
First, given any $U>0$, we have that
\begin{align*}
\mathbb{P}\Big[\|\nabla F_S(x_{\hat t})\|_2>\sqrt{\frac{U}{\tau}} \Big] &\leq \mathbb{P}[T<\tau]+\mathbb{P}[T\|\nabla F_S(x_{\hat t})\|_2^2>U]
\leq \frac14 + \frac{\mathbb{E}[T\|\nabla F_S(x^{\hat t})\|_2^2]}{U} ,
\end{align*}
where the last step follows by Lemma \ref{lem:Run_time_Markov} and Chebyshev's inequality, respectively. Next, by definition of $\hat t$ and Proposition \ref{prop:convergence_biased_SGD_nonconvex} (see \Cref{app:nonconvex_biased_SGD}),
\begin{align*}
&\mathbb{E}[(T+1)\|\nabla F(x^{\hat t})\|_2^2] =\mathbb{E}\Big[\sum_{t=0}^{T}\|\nabla F(x^t)\|_2^2 \Big]\\
&\leq  \frac{\gap}{\eta} + \frac{\eta \smooth}{2} \mathbb{E}\Big[ \sum_{t=0}^{T} \|{\cal G}(x^t)\|_2^2\Big]-\mathbb{E}\Big[\sum_{t=0}^{T}\langle \nabla F(x^t),{\cal G}(x^t)-\nabla F(x^t)\rangle\Big] \\
&\leq  \frac{\gap}{\eta} + \frac{\eta \smooth}{2}  \sum_{t=0}^{\infty} \mathbb{E}[\mathbf{1}_{\{T\geq t\}}\|{\cal G}(x^t)\|_2^2] - \sum_{t=0}^{\infty} \mathbb{E}[\mathbf{1}_{\{T\geq t\}}\langle \nabla F(x^t),{\cal G}(x^t)-\nabla F(x^t)\rangle] \\
&\leq \frac{\gap}{\eta} + \frac{\eta \smooth}{2} \sum_{t=0}^{\infty}\mathbb{P}[T\geq t] \mathbb{E}\big(\mathbb{E}[\|{\cal G}(x^t)\|_2^2|{\cal F}_{t-1}]\big) \\ 
&\quad -\sum_{t=0}^{\infty} \mathbb{P}[T\geq t]\mathbb{E}\big(\mathbb{E}[\langle \nabla F(x^t),{\cal G}(x^t)-\nabla F(x^t)|{\cal F}_{t-1}\rangle]\big) \\
&\leq \frac{\gap}{\eta} + \frac{\eta \smooth}{2} \mathbb{E}[T+1]\nu^2 + \mathbb{E}[T+1] \Lip b\\
 &\lesssim \sqrt{\gap\smooth\tau}\nu + \tau \Lip b,
\end{align*}
where the third inequality holds since  $\{T\geq t\}$ is ${\cal F}_{t-1}$-measurable (see the proof of \Cref{prop:stopping_time_regret} for details), and the fourth inequality follows from \Cref{lem:bias_variance_subsampled_DB},  
used the upper bound on $\mathbb{E}[T]$ from Lemma \ref{lem:Run_time_Markov}, and our choice for $\eta$. Selecting $U=C\big(\sqrt{\gap\smooth\tau}\nu+\Lip \tau b\big)$ with $C>0$ sufficiently large, we get
$\mathbb{E}[T\|\nabla F(x^{\hat t})\|_2^2]/U \leq 1/4$,  
concluding the proof.
\end{proof}

%% file: DP_SCO_Sparse_Gradients/s7-boosting.tex
\section{Boosting the Confidence for the Bias-Reduced Stochastic Gradient Method}
\label{sec:boosting}

We conclude by providing a boosting method to amplify the success probability of our bias-reduced method. 
This private boosting method is a particular instance of a private selection method \citep{Liu:2019}, and it is based on running a random number of independent runs of \Cref{alg:subsampled_DB_SGD} with noisy evaluations of their performance. Among the independent runs, we select the best performing one based on the noisy evaluations. This particular implementation sharpens some polylogarithmic factors that would appear for other private selection methods, such as Report Noisy Min \citep{McSherry:2007,Dwork:2014}.

\begin{algorithm}
\cprotect\caption{%
\texttt{Boosting\_Bias-Reduced\_SGD}%
$(S,\varepsilon,\delta,K)$}
\label{alg:boost_subsampled_DB_SGD}
\begin{algorithmic}
\Require Dataset $S\sim{\cal D}^n$, $\varepsilon,\delta>0$ privacy parameters, random stopping parameter $\gamma\in (0,1)$
\State $K=\frac{1}{\gamma}\ln\big(\frac{2}{\delta}\big)$
\For{$k=1,\ldots, K$}
    \State Run Algorithm \ref{alg:subsampled_DB_SGD} with privacy budget $(\varepsilon/12,(\delta/[4K])^2)$, $\hat x_k$ its output and 
    \If{$f(\cdot,z)$ convex}
    \State Set $s_k=[F_{S}(\hat x_k)+\xi_k]$, where $\xi_k\sim \Lap(\lambda)$, and $\lambda = \frac{12B}{n\varepsilon}$.
\Else
    \State Set $s_k=[\|\nabla F_S(\hat x_k)\|_2+\xi_k]$, where  $\xi_k\sim\Lap(\lambda)$, and $\lambda=\frac{24\Lip}{n\varepsilon}$.    
\EndIf
\State Flip a $\gamma$-biased coin: with probability $\gamma$, {\bf return} $\hat x=\hat x_{\hat k}$, where $\hat k=\arg\min_{l\leq k} s_l$ 
\EndFor
\State {\bf Return} $\hat x=\hat x_{\hat K}$, where $\hat K=\arg\min_{k\leq K} s_k$ 
\end{algorithmic}
\end{algorithm}

\begin{theorem} 
Let $\varepsilon,\delta>0$ such that $\delta\leq \varepsilon/10$. Then
\label{thm:accuracy_boosted_DB_SGD}
Algorithm \ref{alg:boost_subsampled_DB_SGD} is $(\varepsilon,\delta)$-DP. 
Let $0<\beta<1$ and $\gamma=\min\{1/2,3\beta/4\}$. In the convex case, \Cref{alg:boost_subsampled_DB_SGD} attains excess risk 
\ifthenelse{\boolean{journal}}{
\[ \mathbb{P}\Big[ F_S(\hat x)-F_S(\opterm) \leq \alpha \Big]  \geq 1-\beta,\]
}{
$\mathbb{P}\Big[ F_S(\hat x)-F_S(\opterm) \leq \alpha \Big]  \geq 1-\beta$, 
}
where
\[ \alpha
\lesssim  \Lip D\, \frac{\sqrt{\ln n}[s\ln(d/s)\ln^3\big(\ln(1/\delta)/[\beta\delta]\big)]^{1/4}}{\sqrt{\varepsilon n}} +  \frac{B}{n\varepsilon}\ln\Big( \frac{1}{\beta} \ln\big(\frac{2}{\delta}\big) \Big). \]
On the other hand, in the nonconvex case, 
\ifthenelse{\boolean{journal}}{
\[ \mathbb{P} \Big[ \|\nabla F_S(\hat x)\|_2\leq \alpha \Big] \geq 1-\beta,\]
}{
$\mathbb{P} \Big[ \|\nabla F_S(\hat x)\|_2^2\leq \alpha \Big] \geq 1-\beta$, 
}
where 
\[ 
\alpha \lesssim 
\Big(\sqrt{\gap\smooth}\Lip\sqrt{\ln(n)\ln\big(\frac{\ln(1/\delta)}{\beta\delta}\big)}+\Lip^2\Big) \frac{[s\ln(d/s)\ln(\ln(1/\delta)/[\beta\delta])]^{1/4}}{\sqrt{\varepsilon n}}
+\frac{\Lip}{n\varepsilon}\ln\Big( \frac{1}{\beta}\ln\big(\frac{2}{\delta}\big) \Big).
\]
\end{theorem}

\input{DP_SCO_Sparse_Gradients/Proofs/thm22accuracy_boosted_DB_SGD}

%% file: DP_SCO_Sparse_Gradients/Proofs/thm22accuracy_boosted_DB_SGD.tex
\begin{proof}
The privacy analysis follows easily from \citep{Liu:2019}. First, by basic composition, we have that for each $k$ the pair $(\hat x_k,s_k)$ is $(\varepsilon_1,\delta_1)$-DP, with $\varepsilon_1=\varepsilon/6$, and $\delta_1=(\delta/[4K])^2$. By \cite[Thm 3.4]{Liu:2019}, the private selection with random stopping used in \Cref{alg:boost_subsampled_DB_SGD} is such that $\hat x$ is $(3\varepsilon_1+3\sqrt{2\delta_1},\sqrt{2\delta_1}K+\delta/2)$-DP; notice that 
\[ 3\varepsilon_1+3\sqrt{2\delta_1}\leq \frac{\varepsilon}{2}+3\sqrt{2} \frac{\delta}{K} \leq \varepsilon,\]
and
\[ \sqrt{2\delta_1}K+\delta/2 \leq \delta,\]
due to our choices of $\varepsilon_1,\delta_1$. This proves that the algorithm is $(\varepsilon,\delta)$-DP.

The accuracy of the algorithm closely follows \cite[Theorem 3.3]{Liu:2019}. 
First, let $\kappa$ be the number of runs the algorithm makes before stopping, and let $\alpha>0$ to be determined. 
Conditioning on $\kappa$
\begin{align*}
\mathbb{P}\big[F_S(\hat x)-F_S(\opterm)>\alpha\big]
&= \sum_{k=1}^K \mathbb{P}\big[F_S(\hat x)-F_S(\opterm)>\alpha\big| \kappa=k \big] \mathbb{P}[\kappa=k ]\\
&= \sum_{k=1}^K \mathbb{P}\big[F_S(\hat x)-F_S(\opterm)>\alpha\big| \kappa=k \big] (1-\gamma)^{k-1}\gamma.
\end{align*}

We will now bound the conditional probability above. By the subexponential tails of the Laplace distribution, we have that letting ${\cal E}:=\{(\forall j\in [\kappa]):\, |\xi_j|\leq\alpha'\}$ (here, $\alpha'>0$ is arbitrary),
\[ \mathbb{P}[{\cal E}^c|\kappa=k]= \mathbb{P}\Big[ (\exists j\in [\kappa]) \,\,|\xi_k|>\alpha'\Big|\kappa=k \Big] \leq  2k\exp\Big\{ -\frac{n\varepsilon \alpha'}{12B} \Big\}. \]

Hence
\[
    \mathbb{P}\Big[ F_S(\hat x)-F_S(\opterm) > \alpha \Big| \kappa=k \Big] 
    \leq  \mathbb{P}\Big[ \big\{F_S(\hat x)-F_S(\opterm) > \alpha \big\} \cap {\cal E} \Big| \kappa=k  \Big]+ \mathbb{P}[{\cal E}^c|\kappa=k].
\]
Next we have
\begin{align*}
     & \mathbb{P}\Big[ \big\{F_S(\hat x)-F_S(\opterm) > \alpha \big\} \cap {\cal E} \Big| \kappa=k  \Big] 
     \leq \mathbb{P}\Big[ \big\{F_{S}(\hat x_{\hat k})+\xi_{\hat k}-F_S(\opterm) > \alpha-\alpha'\big\} \cap {\cal E} \Big| \kappa=k  \Big] \\
     & = \mathbb{P}\Big[ \big\{\min_{k\in [\kappa]} \big[F_S(\hat x_{k})+\xi_{k}\big]-F_S(\opterm) > \alpha-\alpha'\big\} \cap {\cal E} \Big| \kappa=k  \Big]\\
    & \leq \mathbb{P}\Big[ \min_{k\in [\kappa]}\big[ F_S(\hat x_{k})-F_S(\opterm) \big] > \alpha-2\alpha' \Big| \kappa=k \Big] \\
    & \leq \Big(\mathbb{P}\Big[  F_S(\hat x_{1})-F_S(\opterm) > \alpha-2\alpha' \Big] \Big)^k,
\end{align*}
where in the last step we used that the runs are i.i.d..

We now choose $\alpha,\alpha'$ such that $\alpha-2\alpha'=U/\tau$ (where $U,\tau$ are those from \Cref{thm:convex_const_proba}). Hence, 
\begin{align*}
\mathbb{P}\Big[ F_S(\hat x)-F_S(\opterm) > \alpha \Big| \kappa=k\Big] &\leq 2^{-k}+2k\exp\Big\{ -\frac{n\varepsilon\alpha'}{12B} \Big\}.
\end{align*}

We can now bound the failure probability as follows:
\begin{align*}
\mathbb{P}\big[F_S(\hat x)-F_S(\opterm)>\alpha\big] 
&\leq \sum_{k=1}^K \Big( 2^{-k}+ 2K\exp\Big\{ -\frac{n\varepsilon\alpha'}{12B} \Big\}\Big) (1-\gamma)^{k-1}\gamma\\
&= \frac12\frac{\gamma}{1-\gamma^2}+\frac{2}{\gamma}\ln\big(\frac{2}{\delta}\big) \exp\Big\{ -\frac{n\varepsilon\alpha'}{12B}\Big\}\\
&\leq \frac{\beta}{2}+\frac{2}{\gamma}\ln\big(\frac{2}{\delta}\big) \exp\Big\{ -\frac{n\varepsilon\alpha'}{12B}\Big\},
\end{align*}
where in the last step we used that $\gamma=\min\{1/2,3\beta/4\}$. It is clear then that $\alpha'=\frac{12B}{n\varepsilon}\ln\Big( \frac{16}{3\beta^2} \ln\big(\frac{2}{\delta}\big) \Big)$ makes the probability above at most $\beta$. These choices lead to a final bound
\[ \alpha=\frac{U}{\tau}+2\alpha'
\lesssim  \Lip D\, \frac{\sqrt{\ln n}[s\ln(d/s)\ln^3\big(\ln(1/\delta)/[\beta\delta]\big)]^{1/4}}{\sqrt{\varepsilon n}} +  \frac{B}{n\varepsilon}\ln\Big( \frac{1}{\beta} \ln\big(\frac{2}{\delta}\big) \Big). \]

For the nonconvex case, we need to replace $B$ by $2\Lip$ in the Laplace concentration bound. Further, we consider the event $\{\|\nabla F(\hat x_k)\|_2>\alpha\}$ (as opposed to the optimality gap event). This implies that we need to set $\alpha >0$ such that $\alpha-2\alpha'\geq \sqrt{U/\tau}$ from \Cref{thm:nonconvex_const_proba}. This leads to 
\[
\mathbb{P}\Big[ \|F_S(\hat x)\|_2 > \alpha \Big] \leq \sum_{k=1}^K \Big(2^{-k}+2K\exp\Big\{ -\frac{n\varepsilon\gamma}{24\Lip} \Big\}\Big) (1-\gamma)^{k-1}\gamma.
\]
The rest of the derivations are analogous.
\end{proof}

%% file: DP_SCO_Sparse_Gradients/s5-dp-erm-sco-output-perturbation.tex
\section{DP Convex Optimization with Sparse Gradients via Regularized Output Perturbation}

\label{sec:Output_Perturbation}

We conclude our work introducing another class of algorithms that attains nearly optimal rates for approximate-DP ERM and SO in the convex setting. These algorithms are based on solving a regularized ERM problem and privatizing its output by an output perturbation method. The main innovation of this technique is that we reduce the noise error by a $\|\cdot\|_{\infty}$-projection. This type of projection leverages the concentration of the noise in high-dimensions. We carry out an analysis that also leverages the convexity of the risk and the gradient sparsity to obtain these rates. The full description is included in \Cref{alg:output-pert-inf}. 
\ifthenelse{\boolean{journal}}{}{We defer missing proofs from this section, as well as additional results, to
\Cref{app:missing_proofs_output_perturbation}.
}

\begin{algorithm}
\cprotect\caption{%
\texttt{Output\_Perturbation}%
}\label[algorithm]{alg:output-pert-inf}
\begin{algorithmic}
\Require Dataset $S = (z_1, \dots, z_n) \in {\cZ}^n$, $\varepsilon,\delta\geq 0$ privacy params., $f(\cdot,z)$ $\Lip$-Lipschitz convex function (if $\delta=0$ further assume $\smooth$-smooth) with $s$-sparse gradient, $\lambda\geq 0$ regularization param.
\State Let  
$\optregerm = \argmin_{x \in \cX} F^{\lambda}_{S}(x)$, where $F^{\lambda}_{S}(x) := \big[ F_S(x) + \frac{\lambda}{2} \|x\|_2^2 \big]$
\State $\tilde x=\optregerm+\xi$, with $\xi\sim\begin{cases}
\Lap(\sigma)^{\otimes d} &\text{with } \sigma=\frac{2\sqrt{2s}\Lip}{\lambda \varepsilon n}\big(\frac{2\smooth}{\lambda}+1\big) \text{ if } \delta = 0\,,\\[1mm]
{\cN}(0,\sigma^2I) &\text{with } 
\sigma^2=\frac{8\Lip^2\ln(1.25/\delta)}{[\lambda\varepsilon n]^2} \text{ if } \delta > 0\,.\\
\end{cases}$
\State \Return $\hat x = \argmin_{x \in \cX} \|x - \tilde x\|_{\infty}$ (breaking ties arbitrarily)
\end{algorithmic}
\end{algorithm}

\ifthenelse{\boolean{journal}}{
\subsection{DP-ERM via Output Perturbation}
\label{sec:dp_erm_output_pert}

}{}

\begin{theorem} \label[theorem]{thm:output-pert-improved}
Consider an ERM problem under assumptions:  initial distance  (\Cref{assumpt:init_dist}), convexity (\Cref{assumpt:cvx}), Lipchitzness (\Cref{assumpt:Lipschitz}) and gradient sparsity (\Cref{assumpt:sparse_grad}). Then, \Cref{alg:output-pert-inf} is $(\varepsilon,\delta)$-DP, and it satisfies the following excess risk guarantees, for any $0<\beta<1$:
\begin{itemize}
\item If $\delta=0$, and under the additional assumption of smoothness \ref{assumpt:smoothness} and unconstrained domain, ${\cX}=\RR^d$, then selecting $\lambda=\Big( \frac{\Lip^2\smooth}{D^2} \frac{s\log(d/\beta)}{\varepsilon n} \Big)^{1/3}$, it holds with probability $1-\beta$ that
\[\textstyle
F_S(\hat x)- F_S(\opterm) \lesssim \Lip^{2/3}\smooth^{1/3}D^{4/3}\Big(\frac{s\log(d/\beta)}{\varepsilon n}\Big)^{1/3}.
\]
\item If $\delta>0$ then selecting $\lambda = \frac{\Lip}{D} \cdot \frac{[s \log(1/\delta)\log(d/\beta)]^{1/4}}{\sqrt{\eps n}}$, we have with probability $1-\beta$ that
\begin{align*}
\textstyle F_S(\hat x) - F_S(\opterm) \lesssim \Lip D \cdot \frac{(s \log(1/\delta)\log(d/\beta))^{1/4}}{\sqrt{\eps n}}.
\end{align*}
\end{itemize}
\end{theorem}

\begin{remark}
For approximate-DP, the theorem above can also be proved if we replace assumption (Item \ref{assumpt:init_dist}) by the diameter assumption (Item \ref{assumpt:diameter}). On the other hand, for the pure-DP case it is a natural question whether the smoothness assumption is essential. In Appendix \ref{sec:sparse_exp_mech}, we provide a version of the exponential mechanism that works without the smoothness and unconstrained domain assumptions. This algorithm is inefficient and it does require an structural assumption on the feasible set, but it illustrates the possibilities of more general results in the pure-DP  setting.
\end{remark}

\ifthenelse{\boolean{journal}}{
\input{DP_SCO_Sparse_Gradients/Proofs/thm07output-pert-improved}

\input{DP_SCO_Sparse_Gradients/s5.1-sparse-exp-mech}
}{}


\ifthenelse{\boolean{journal}}{
\subsection{DP-SCO via Output Perturbation}
\label{sec:dp_sco_output_pert}
}{}

We note that the proposed output perturbation approach (\Cref{alg:output-pert-inf}) leads to nearly optimal population risk bounds for approximate-DP, by a different tuning of the regularization parameter $\lambda$.

\begin{theorem} \label[theorem]{thm:output_pert_approx_DP_SCO}
Consider a problem \eqref{eqn:SO} under bounded initial distance (\Cref{assumpt:init_dist}) (or bounded diameter, \Cref{assumpt:diameter}, if $\delta>0$),  convexity (\Cref{assumpt:cvx}), Lipschitzness (\Cref{assumpt:Lipschitz}), bounded range (\Cref{assumpt:loss_range}), 
and gradient sparsity (\Cref{assumpt:sparse_grad}). Then, \Cref{alg:output-pert-inf} is $(\varepsilon,\delta)$-DP, and for $0<\beta<1$,\vspace{-0.1cm}
\begin{itemize}[itemsep=0pt,leftmargin=6mm]
\item If $\delta=0$, and under the additional assumption of smoothness \ref{assumpt:smoothness} and unconstrained domain, ${\cX}=\RR^d$. Selecting $\lambda= 
\Big(\frac{\Lip^2\smooth}{D^2} \frac{s\log(d/\beta)}{\varepsilon n} \Big)^{1/3}$,  
then with probability $1-\beta$\vspace{-0.2cm}
\[\textstyle
F_S(\hat x)- F_S(\optpop) ~\lesssim~ \Lip^{2/3}\smooth^{1/3}D^{4/3}\Big(\frac{s\log(d/\beta)}{\varepsilon n}\Big)^{1/3}+B\sqrt{\frac{\ln(1/\beta)}{n}}. \vspace{-0.2cm}\]
\item If $\delta>0$. Selecting $\lambda=\frac{\Lip}{D}\Big(\frac{\ln(n)\ln(1/\beta)}{n}+\frac{\sqrt{s\ln(1/\delta)\ln(d/\beta)}}{\varepsilon n}\Big)^{1/2}$, then with probability $1-\beta$\vspace{-0.2cm}
\[\textstyle
F_{\cD}(\hat x)-F_{\cD}(\optpop) 
~\lesssim~
\Lip D \frac{[s\ln(1/\delta)\log(d/\beta)]^{1/4}}{\sqrt{\varepsilon n}}+(\Lip D\sqrt{\ln n}+B)\sqrt{\frac{\ln(1/\beta)}{n}}.
\vspace{-0.1cm}\] 
\end{itemize}
\end{theorem}

\ifthenelse{\boolean{journal}}{
\input{DP_SCO_Sparse_Gradients/Proofs/thm09output_pert_approx_DP_SCO}
}{}

%% file: DP_SCO_Sparse_Gradients/Proofs/thm07output-pert-improved.tex
\begin{proof}
We proceed by cases:

\begin{itemize}
\item {\bf Case $\delta=0$.}  First, we prove that privacy of the algorithm. To do this, we first establish a bound on the $\ell_1$-sensitivity of the (quadratically) regularized ERM. Note that the first-order optimality conditions in this case correspond to
\[ \optregerm = -\frac{1}{\lambda} \nabla F_S(\optregerm).\]
Therefore, if $S\simeq S^{\prime}$, where $S=(z_1,\ldots,z_n)$ and  $S=(z_1^{\prime},\ldots,z_n^{\prime})$ only differ in one entry,
\begin{align*}
\|\optregerm-\optregermp\|_1 &\leq \frac{1}{\lambda}\|\nabla F_S(\optregerm)-\nabla F_{S^{\prime}}(\optregermp)\|_1\\
&\leq \frac{1}{\lambda n}\sum_{i=1}^n\|\nabla f(\optregerm,z_i)-\nabla f(\optregermp,z_i^{\prime})\|_1\\
&\leq \frac{1}{\lambda n}\Big[ (n-1)\sqrt{2s}\smooth\|\optregerm-\optregermp\|_2 +2\sqrt{2s}\Lip\Big] \\
&\leq \frac{1}{\lambda n}\Big(4\sqrt{2s}\smooth\Lip\frac{n-1}{\lambda n}+2\sqrt{2s}\Lip\Big)\\
&\leq \frac{2\sqrt{2s}\Lip}{\lambda n}\big(\frac{2\smooth}{\lambda}+1\big).
\end{align*}
Above, in the third inequality we used the gradient sparsity \ref{assumpt:sparse_grad}, and the smoothness \ref{assumpt:smoothness}, assumptions. In the fourth inequality we used that the regularized ERM has $\ell_2$-sensitivity $\frac{4\Lip}{\lambda n}$ \cite{Bousquet:2002,SSSS09,ChaudhuriMS11}. We conclude the privacy then by \Cref{fact:privacy}(a).

We also remark that by \Cref{fact:accuracy}(a)-(i), 
$\|\xi\|_{\infty}\lesssim \frac{\Lip\sqrt{s}\ln(d/\beta)}{\lambda n\varepsilon}\Big(\frac{\smooth}{\lambda}+1\Big)$, with probability $1-\beta$.

\item {\bf Case $\delta>0$.}   
The privacy guarantee follows from the fact that the $\ell_2$-sensitivity of $\optregerm$ is $\frac{4\Lip}{\lambda n}$ \cite{Bousquet:2002,SSSS09,ChaudhuriMS11}, together with \Cref{fact:privacy}(b).

Moreover, by \Cref{fact:accuracy}(b)-(i),  
$\|\xi\|_{\infty}\lesssim \frac{\Lip\sqrt{\ln(d/\beta)}}{\lambda n\varepsilon}$, with probability $1-\beta$.
\end{itemize}

We continue with the accuracy analysis, making a unified presentation for both pure and approximate DP. First, by the optimality conditions of the regularized ERM,
\begin{equation} \label{eq:diff-from-reg}
F_S(\optregerm)-F_S(\opterm) \leq \frac{\lambda}{2}\|\opterm\|^2\leq \frac{\lambda}{2} D^2. 
\end{equation}
We need the following key fact, which follows by the definitions of $\hat x$ and $\tilde x$,
\begin{equation}
\label{eqn:proj_infty}
\|\hat x-\optregerm\|_{\infty} \leq \|\hat x-\tilde x\|_{\infty}+\|\tilde x-\optregerm\|_{\infty}
\leq 2\|\xi\|_{\infty}.
\end{equation}
Using these two bounds, we proceed as follows
\begin{align*}
F_S(\hat x)-F_S(\opterm) 
&\leq F_S(\hat x)-F_S(\optregerm)+\frac{\lambda}{2} D^2
\leq \langle \nabla F_S(\hat x),\hat x-\optregerm\rangle+\frac{\lambda}{2}D^2\\
&\leq \|\nabla F_S(\hat x)\|_1\|\hat x-\optregerm\|_{\infty}+\frac{\lambda}{2}D^2\\
&\leq \sqrt{2s}\Lip \|\xi\|_{\infty}+\frac{\lambda}{2}D^2, 
\end{align*}
where the second inequality follows by convexity of $F_S$, and the fourth one by the gradient sparsity assumption and \eqref{eqn:proj_infty}. 

The conclusion follows by plugging in the respective bounds of $\lambda$ and $\|\xi\|_{\infty}$, for both pure and approximate DP cases.

\end{proof}

%% file: DP_SCO_Sparse_Gradients/s5.1-sparse-exp-mech.tex
\subsection{A Pure DP-ERM Algorithm for Nonsmooth Losses}

\label[section]{sec:sparse_exp_mech}

We now prove that the rates of pure DP-ERM in the convex case above can be obtained without the smoothness assumption, albeit with an inefficient algorithm. This algorithm is based on the exponential mechanism, and it leverages the fact that the convex ERM with sparse gradient always has an approximate solution which is sparse. This result requires an additional assumption on the feasible set:
\begin{eqnarray}\label{assump:sparsifiable_set}
 \big( x\in {\cal X} \wedge P\subseteq [d]\big) \quad\Longrightarrow\quad x|_P \in {\cal X},
\end{eqnarray}
where $x|_P\in \RR^d$ is the vector such that $x_{P,j}=x_j$ if $j\in P$, and $x_{P,j}=0$ otherwise. We will say that ${\cal X}$ is sparsifiable if \eqref{assump:sparsifiable_set} holds. Note this property holds e.g.~for $\ell_p$-balls centered at the origin.

\begin{lemma} \label{lem:sparsify_ERM}
Let ${\cal X}$ be a convex sparsifiable set. Consider the problem \eqref{eqn:ERM} under convexity (\Cref{assumpt:cvx}), bounded diameter (\Cref{assumpt:diameter}), Lipschitzness (\Cref{assumpt:Lipschitz}) and gradient sparsity (\Cref{assumpt:sparse_grad}), assumptions. If $\opterm$ is an optimal solution of \eqref{eqn:ERM} and $\tau>0$, then there exists $\tilde x\in {\cal X}$ such that $\|\tilde x\|_{0}\leq 1/\tau^2$, and
\[ F_S(\tilde x)-F_S(\opterm)\leq \Lip\sqrt{s}\tau.\]
\end{lemma}

\begin{proof}
Let $\tilde x\in \RR^d$ be defined as
\[\tilde x_j =
\left\{ 
\begin{array}{ll}
x_j & \mbox{ if } |x_{S,j}^{\ast}|\geq \tau\\
0 & \mbox{ otherwise.}
\end{array}
\right.\]
Note that $\tilde x\in {\cal X}$ since $\opterm\in {\cal X}$ and ${\cal X}$ is sparsifiable. 
Now we note that 
\[\|\tilde x\|_0 \leq \sum_{j:\,|x_{S,j}^{\ast}|\geq \tau} \frac{(x_{S,j}^{\ast})^2}{\tau^2} \leq \frac{1}{\tau^2}. \]
Finally, for the accuracy guarantee, we use convexity as follows,
\begin{align*}
F_S(\tilde x)-F_S(\opterm) 
&\leq \langle \nabla F_S(\tilde x), \hat x-\opterm\rangle \\
&\leq \|\nabla F_S(\tilde x)\|_1 \|\tilde x-\opterm\|_{\infty} \\
&\leq \Lip\sqrt{s}\tau,
\end{align*}
where in the last step we used that $\nabla f(\hat x,z_i)\in \sparsevec$ and the definition of $\tilde x$.
\end{proof}

We present now the {\em sparse exponential mechanism}, which uses the result above to approximately solve \eqref{eqn:ERM} with nearly dimension-independent rates.

\begin{algorithm}
\cprotect\caption{\verb|Sparse_Exponential_Mechanism|}\label{alg:sparse_exponential_mech}
\begin{algorithmic}
\Require Dataset $S = \{z_1, \dots, z_n\} \subseteq {\cal Z}$, $\varepsilon$ privacy parameter, $f(\cdot,z)$ $\Lip$-Lipschitz convex function with $s$-sparse gradients and range bounded by $B$, $0<\beta<1$ confidence parameter
\State Let $\tau>0$ be such that 
$\frac{\tau^3}{\ln(d/[\tau\beta])}=\frac{\Lip\sqrt{s}\varepsilon n}{B} $
\State Let ${\cal N}_{\tau}$ be a $\tau$-net of $1/\tau^2$-sparse vectors over ${\cal X}$ with $|{\cal N}_{\tau}|\leq \binom{d}{1/\tau^2} \big(\frac{3}{\tau}\big)^{1/\tau^2}$
\State Let $\hat x$ be a random variable supported on ${\cal N}_{\tau}$ such that $\mathbb{P}[\hat x=x]\propto \exp\big\{-\frac{B}{\varepsilon n} F_S(x) \big\}$
\State {\bf Return} $\hat x$
\end{algorithmic}
\end{algorithm}

\begin{remark}
The bound on $|{\cal N}_{\tau}|$ claimed in \Cref{alg:sparse_exponential_mech} follows from a standard combinatorial argument (e.g.~\cite{Vershynin:2009}). Moreover, it follows that $|{\cal N}_{\tau}|\lesssim \Big(\frac{d}{\tau}\Big)^{1/\tau^2}$.
\end{remark}

\begin{theorem} \label{thm:sparse_exp_mech}
Let ${\cal X}$ be a convex sparsifiable set. Consider a problem \eqref{eqn:ERM} under bounded diameter (\Cref{assumpt:diameter}), convexity (\Cref{assumpt:cvx}), bounded range (\Cref{assumpt:loss_range}), Lipschitzness (\Cref{assumpt:Lipschitz}) and gradient sparsity (\Cref{assumpt:sparse_grad}), assumptions. Then \Cref{alg:sparse_exponential_mech} satisfies with probability $1-
\beta$
\[ F_S(\hat x)-F_S(\opterm) \lesssim  \Lip^{2/3}B^{1/3} \Big(\frac{s}{\varepsilon n}\ln\Big(\frac{\Lip\sqrt{s}\varepsilon n}{B}\frac{d}{\beta} \Big)\Big)^{1/3}. \]
\end{theorem}

\begin{proof}
Let $\tilde x$ be the vector whose existence is guaranteed by \Cref{lem:sparsify_ERM}. 
By the high-probability guarantee of the exponential mechanism \cite{Dwork:2014} with probability $1-\beta$,
\[  
F_S(\hat x)-F_S(\tilde x) \leq\frac{B}{\varepsilon n}\Big(\ln|{\cal N}_{\tau}|+\ln(1/\beta)\Big)
\lesssim \frac{B}{\varepsilon n}\frac{\ln\big(\frac{d}{\tau\beta}\big)}{\tau^2}.
\]
Hence, using \Cref{lem:sparsify_ERM} with the upper bound above,
\begin{align*}
F_S(\hat x)-F_S(\opterm)
&\leq F_S(\hat x)-F_S(\tilde x)+F_S(\tilde x)-F_S(\opterm) \\
&\lesssim \frac{B}{\varepsilon n}\frac{\ln(d/[\tau\beta])}{\tau^2}+\Lip\sqrt{s}\tau \\
&\lesssim \Big(\Lip^2B \frac{s}{\varepsilon n}\ln\Big(\frac{\Lip\sqrt{s}\varepsilon n}{B}\big(\frac{d}{\beta}\big)^3 \Big)\Big)^{1/3},
\end{align*}
where we used our choice of $\tau$.
\end{proof}

%% file: DP_SCO_Sparse_Gradients/Proofs/thm09output_pert_approx_DP_SCO.tex
\begin{remark}
Note first that in the proof below we are not addressing the privacy of \Cref{alg:output-pert-inf}, as this has already been proven in \Cref{thm:output-pert-improved}.

On the other hand, note that the same proof below --using the in-expectation generalization guarantees of uniformly stable algorithms \citep{Bousquet:2002}-- provides a sharper upper bound for the expected excess risk for the pure and approximate DP cases, which would hold w.p.~$1-\beta$ over the algorithm internal randomness
\begin{eqnarray*}
\mathbb{E}_S[F_{\cal D}(\hat x)-F_{\cal D}(\optpop)]
&\lesssim& \Lip^{2/3}\smooth^{1/3}D^{4/3}\Big(\frac{s\log(d/\beta)}{\varepsilon n}\Big)^{1/3},\\
\mathbb{E}_S[F_{\cal D}(\hat x)-F_{\cal D}(\optpop)]
&\lesssim& \Lip D \frac{[s\ln(1/\delta)\log(d/\beta)]^{1/4}}{\sqrt{\varepsilon n}}. 
\end{eqnarray*}
\end{remark}

\begin{proof}
Using the $\ell_2$-sensitivity of $\optregerm$, $\Delta_2=\frac{4\Lip}{\lambda n}$, we have the following generalization bound \citep{Bousquet:2020}: with probability $1-\beta/2$
\[   F_{\cal D}(\optregerm) -  F_{S}(\optregerm) \lesssim \frac{\Lip^2}{\lambda n}\ln( n) \ln\Big(\frac{1}{\beta}\Big)+B\sqrt{\frac{\ln\big(\frac{1}{\beta}\big)}{n}}=:\gamma.\]

The bound of \eqref{eq:diff-from-reg} can be obviously modified by comparison with the population risk minimizer, $\optpop$: in particular, the event above\footnote{We also need concentration to upper bound  $F_S(\optpop)-F_{\cal D}(\optpop)$. However, this is easy to do by e.g.~Hoeffding's inequality, leading to a bound $\lesssim \gamma$.} implies that
\[ 
F_{\cal D}(\optregerm)-F_{\cal D}(\optpop) \lesssim F_S(\optregerm)-F_S(\optpop)+\gamma
\leq \frac{\lambda}{2}\|\optpop\|_2^2 + \gamma
\lesssim \lambda D^2 +\gamma. 
\]
On the other hand, the bound \eqref{eqn:proj_infty} works exactly as in the proof of \Cref{thm:output-pert-improved}. 
Hence, we have that with probability $1-\beta/2$,
\begin{align*}
F_{\cal D}(\hat x)-F_{\cal D}(\optpop)
&\lesssim F_{\cal D}(\hat x)-F_{\cal D}(\optregerm)+\lambda D^2+\gamma\\
&\lesssim \langle \nabla F_{\cal D}(\hat x),\hat x -\optregerm\rangle+\lambda D^2+\gamma\\
&\lesssim 2\Lip\sqrt{s}\|\xi\|_{\infty}+\frac{\Lip^2}{\lambda n}\ln( n) \ln\Big(\frac{1}{\beta}\Big) +\lambda D^2+\frac{B}{\sqrt{n}}\sqrt{\ln\big(\frac{1}{\beta}\big)},
\end{align*}
where in the last step we used that $\|\nabla F_{\cal D}(\hat x)\|_1=\|\mathbb{E}_z[\nabla f(\hat x,z)]\|_1\leq \mathbb{E}_z[\|\nabla f(\hat x,z)\|_1] \leq L\sqrt{s}$ (the last step which follows by the gradient sparsity), inequality \eqref{eqn:proj_infty}, and the definition of $\gamma$.

We proceed now by separately studying the different cases for $\delta$:
\begin{itemize}
\item {\bf Case $\delta=0$.} The bound above becomes
\[ F_{\cal D}(\hat x)-F(\optpop)\lesssim 
\frac{\Lip^2}{\lambda n}\Big( \frac{s\ln(d/\beta)}{\varepsilon}\big(\frac{\smooth}{\lambda}+1\big)+\ln n \ln(1/\beta) \Big)+\lambda D^2+B\sqrt{\frac{\ln(1/\beta)}{n}}. \]
Our choice of $\lambda$ provides the claimed bound.
\item {\bf Case $\delta>0$.} Here, the upper bound takes the form
\[  F_{\cal D}(\hat x)-F(\optpop)\lesssim \frac{\Lip^2}{\lambda n}\Big( \frac{\sqrt{s\ln(d/\beta)\ln(1/\delta)}}{\varepsilon}+\ln(n)\ln(1/\beta) \Big)+\lambda D^2+B\sqrt{\frac{\ln(1/\beta)}{n}}. \]
The proposed value of $\lambda$ leads to the bound below that holds with probability $1-\beta$,
\begin{align*}
    F_{\cal D}(\hat x)-F_{\cal D}(\optpop)
    &\lesssim 
    B\sqrt{\frac{\ln(1/\beta)}{n}}+\Lip D \sqrt{\frac{\ln n \ln(1/\beta)}{n}+\frac{\sqrt{s\ln(1/\delta)\log(d/\beta)}}{\varepsilon n}  }  \\
    &\lesssim 
    (\Lip D\sqrt{\ln n}+B)\sqrt{\frac{\ln(1/\beta)}{n}}+
    \Lip D \frac{[s\ln(1/\delta)\log(d/\beta)]^{1/4}}{\sqrt{\varepsilon n}}.
\end{align*} 
\end{itemize}
\end{proof}

%% file: DP_SCO_Sparse_Gradients/a1-appendix.tex
\appendix

\section{Auxiliary Privacy Results} \label{app:auxiliary_results}

\ifthenelse{\boolean{journal}}{}{\input{DP_SCO_Sparse_Gradients/Proofs/facts_noise_addition}}

We note the existence of packing sets of sparse vectors (e.g.,~\citet{Vershynin:2009,Candes:2011}).
Denote by $\cC_s^d$ the set of all $s$-sparse vectors in $\{0, 1/\sqrt{s}\}^d$; note that $\cC_s^d \subseteq \sparsevec$. 

\input{DP_SCO_Sparse_Gradients/Proofs/lem02sparse_codes}

For completeness, we provide a classical dataset bootstrapping argument used for DP mean estimation lower bounds \citep{Bassily:2014}. Whereas in the original reference this bootstrapping is achieved by appending dummy vectors which mutually cancel out with the goal of maintaining the structure of vectors, we simply append zero vectors as dummies as we do not need to satisfy an exact sparsity pattern.

\input{DP_SCO_Sparse_Gradients/Proofs/lem03bootstrap-privacy-lb}

\input{DP_SCO_Sparse_Gradients/Proofs/lem04packing_DP_LB}

We will make use of the following {\em fully adaptive composition} property of DP, which informally states that for a prescribed privacy budget, a composition of (adaptively chosen) mechanisms whose privacy parameters are predictable, if we stop the algorithm before the (predictable) privacy budget is exhausted, the result of the full transcript is DP.

\begin{theorem}[$(\eps,\delta)$-DP Filter, \cite{Whitehouse:2023}]
\label{thm:DP_Filter}
Suppose $(\cA_t)_{t\geq 0}$ is a sequence of algorithms such that, for any $t\geq 0$, $\cA_t$ is $(\eps_t,\delta_t)$-DP, conditionally on $(\cA_{0:t-1})$ (in particular, $(\varepsilon_t,\delta_t)_t$ is $({\cal A}_t)_t$-predictable). Let $\eps>0$ and $\delta=\delta^{\prime}+\delta^{\prime\prime}$ be the target privacy parameters such that $\delta^{\prime}>0,\delta^{\prime\prime}\geq0$. Let
\[\eps_{[0:t]}:= \sqrt{2\ln\big(\frac{1}{\delta^{\prime}}\big)\sum_{s=0}^{t}\eps_s^2} +\frac{1}{2}\sum_{s=0}^{t}\eps_s^2, \quad\mbox{ and } \qquad
\delta_{[0:t]}:= \sum_{s=0}^{t}\delta_s,
\]
and define the stopping time
\[ T((\eps_t,\delta_t)_t):=\inf\Big\{t:\, \eps< \eps_{[0:t+1]}\Big\} \wedge \inf\Big\{t:\, \delta^{\prime\prime}< \delta_{[0:t+1]} \Big\}. \]
Then, the algorithm $\cA_{0:T(\cdot)}(\cdot)$ is $(\eps,\delta)$-DP, where $T(x)=T\big((\eps_t(x),\delta_t(x)\big)_{t\geq 0}$.
\end{theorem}

\ifthenelse{\boolean{journal}}{}
{
    \section{Missing Proofs from \Cref{sec:MeanEstimationUB}} \label{app:Mean_Estimation}

    \subsection{Proof of Lemma \ref{lem:proj-main}}

\input{DP_SCO_Sparse_Gradients/Proofs/lem01-proj-main}
    \subsection{Proof of Theorem \ref{thm:proj_mech_pure_DP}}

\input{DP_SCO_Sparse_Gradients/Proofs/thm01proj_mech_pure_DP}
    \subsection{Proof of Theorem \ref{thm:proj_mech_approx_DP}}

\input{DP_SCO_Sparse_Gradients/Proofs/thm2_proj_mech_approx_DP}
    \subsection{Sharper DP Mean Estimation Upper Bounds via Compressed Sensing}

    \label{sec:compressed_sensing}

\input{DP_SCO_Sparse_Gradients/compressed_sensing_mean_estimation}
}

\ifthenelse{\boolean{journal}}{}{
    \section{Missing Proofs from \Cref{sec:LowerBounds}}

    \label{app:LowerBounds}

    \subsection{Proof of \Cref{lem:block_diagonal_bootstrap}}

\input{DP_SCO_Sparse_Gradients/Proofs/lem07block_diagonal_bootstrap}
    \subsection{Proof of \Cref{thm:LB_pure_DP_sparse}}

\input{DP_SCO_Sparse_Gradients/Proofs/thm08LB_pure_DP_sparse}
    \subsection{Proof of \Cref{thm:LB_pure_DP_sparse_bootstrap}}

\input{DP_SCO_Sparse_Gradients/Proofs/thm04LB_pure_DP_sparse_bootstrap}
    \subsection{Proof of \Cref{thm:LB_approx_DP_sparse}}
    \input{DP_SCO_Sparse_Gradients/Proofs/thm06LB_approx_DP_sparse}
}

\section{Analysis of Biased SGD}

\label{app:biased_SGD}

Given the heavy-tailed nature of our estimators, our guarantees for a single run of SGD with bias-reduced first-order oracles only yields constant probability guarantees. Here we prove pathwise bounds that facilitate such analyses.

\subsection{Excess Empirical Risk: Convex Case}

First, we provide a path-wise guarantee for a run of SGD with a biased oracle. Importantly, this guarantee is made of a method which runs for a {\em random number of steps}.

\begin{proposition} \label{prop:random_stop_SGD}
    Let $({\cal F}_t)_t$ be the natural filtration, and $T$ be a random time. Let $(x^t)_t$ be the trajectory of projected SGD with deterministic stepsize sequence $(\eta_t)_t$, and (biased) stochastic first-order oracle ${\cal G}$ for a given function $F$. If $x^{\ast}\in\arg\min\{F(x):x\in{\cal X}\}$, then the following event holds almost surely
    \[ 
    \sum_{t=0}^{T}[F(x^t)-F(x^{\ast})] \leq \frac{1}{2\eta_t}\|x^0-x^{\ast}\|^2+\sum_{t=0}^{T}\big[ \frac{\eta_t}{2}\|{\cal G}(x^t)\|^2+\langle \nabla F(x^t)-{\cal G}(x^t),x^t-x^{\ast}\rangle \big]. 
    \]
\end{proposition}

\begin{proof}
By convexity
\begin{align*}
    F(x^t)-F(x^{\ast}) &\leq \langle \nabla F(x^t),x^t-x^{\ast}\rangle 
    = \underbrace{ \langle \nabla F(x^t)-{\cal G}(x^t),x^t-x^{\ast}\rangle }_{:=b_t}+\langle {\cal G}(x^t), x^t-x^{\ast}\rangle \\
    &\leq b_t+\langle {\cal G}(x^t), x^t-x^{t+1}\rangle+\langle {\cal G}(x^t), x^{t+1}-x^{\ast}\rangle \\ 
    &\leq b_t+\frac{\eta_t}{2}\|{\cal G}(x^t)\|^2+\frac{1}{2\eta_t}\|x^t-x^{t+1}\|^2+\langle \nabla {\cal G}(x^t), x^{t+1}-x^{\ast}\rangle \\ 
    &\stackrel{(\ast)}{\leq} b_t+\frac{\eta_t}{2}\|{\cal G}(x^t)\|^2+\frac{1}{2\eta_t}\|x^t-x^{t+1}\|^2+
    \frac{1}{\eta_t}\langle x^{t+1}-x^t,x^{\ast}-x^{t+1}\rangle \\ 
    &= b_t+\frac{\eta_t}{2}\|{\cal G}(x^t)\|^2+\frac{1}{2\eta_t}\|x^t-x^{\ast}\|^2-
    \frac{1}{2\eta_t}\|x^{t+1}-x^{\ast}\|^2,
\end{align*}
where the second inequality follows by the Young inequality, and step $(\ast)$ we used the optimality conditions of the projected SGD step:
\[ \langle \eta_t{\cal G}(x^t)+[x^{t+1}-x^t], x-x^{t+1}\rangle \geq 0\quad(\forall x\in {\cal X}).\]
Therefore, summing up these inequalities, we obtain 
\begin{align*}
\sum_{t=0}^{T}[F(x^t)-F(x^{\ast})] \leq \frac{1}{2\eta_0}\|x^0-x^{\ast}\|^2+\sum_{t=0}^{T}\Big[\frac{\eta_t}{2}\|{\cal G}(x^t)\|^2+b_t\Big].
\end{align*}
Plugging in the definition of $b_t$ proves the result.
\end{proof}

\subsection{Stationary Points: Nonconvex Case}

\label{app:nonconvex_biased_SGD}

\begin{proposition} \label{prop:convergence_biased_SGD_nonconvex}
Let $F$ satisfy \ref{assumpt:smoothness}, and let ${\cal G}$ be a biased first-order stochastic oracle for $F$. 
Let $(x^t)_t$ be the trajectory of SGD with oracle ${\cal G}$, constant stepsize $0<\eta\leq 1/[2\smooth]$, and initialization $x^0$ such that $F(x^0)-\min_{x\in \mathbb{R}^d}F(x) \leq \gap$. Let $T$ be a random time. Then the following event holds almost surely
\[\sum_{t=0}^{T}\|\nabla F(x^t)\|_2^2 
\leq  \frac{\gap}{\eta} + \frac{\eta \smooth}{2} \sum_{t=0}^{T} \|{\cal G}(x^t)\|_2^2-\sum_{t=0}^{T}\langle \nabla F(x^t),{\cal G}(x^t)-\nabla F(x^t)\rangle  \]
\end{proposition}

\begin{proof}
By smoothness of $f$, we have
\begin{align*}
F(x^{t+1})-F(x^t) 
&\leq -\eta\langle \nabla F(x^t),{\cal G}(x^t)\rangle +\frac{\eta^2 \smooth}{2} \|{\cal G}(x^t)\|_2^2\\
&\leq -\eta\|\nabla F(x^t)\|_2^2-\eta\langle \nabla F(x^t),{\cal G}(x^t)-\nabla F(x^t)\rangle + \frac{\eta^2 \smooth}{2} \|{\cal G}(x^t)\|_2^2.
\end{align*}
Therefore, 
\begin{align*} 
\sum_{t=0}^{T}\|\nabla F(x^t)\|_2^2 
&\leq \frac{F(x^0)-F(x^{T+1})}{\eta} -\sum_{t=0}^{T}\langle \nabla F(x^t),{\cal G}(x^t)-\nabla F(x^t)\rangle + \frac{\eta \smooth}{2} \sum_{t=0}^{T} \|{\cal G}(x^t)\|_2^2 \\
& \leq \frac{\gap}{\eta} -\sum_{t=0}^{T}\langle \nabla F(x^t),{\cal G}(x^t)-\nabla F(x^t)\rangle + \frac{\eta \smooth}{2} \sum_{t=0}^{T} \|{\cal G}(x^t)\|_2^2.
\end{align*}
\end{proof}

\ifthenelse{\boolean{journal}}{}{
    \section{Missing proofs from \Cref{sec:bias_reduction}}

    \label{app:sec_bias_reduction_proofs}

    \subsection{Proof of \Cref{lem:privacy_subsampled_DB}}

\input{DP_SCO_Sparse_Gradients/Proofs/lem16privacy_subsampled_DB}
     \subsection{Proof of Lemma \ref{lem:Run_time_Markov}}

\input{DP_SCO_Sparse_Gradients/Proofs/lem15Run_time_Markov}
     \subsection{Excess Empirical Risk in the Convex Setting}

    \label{sec:DP_ERM_bias_reduction}

\input{DP_SCO_Sparse_Gradients/s6.3.1convex_ERM_bias_reduction}\input{DP_SCO_Sparse_Gradients/s6.3.2nonconvex_ERM_bias_reduction}
}

\ifthenelse{\boolean{journal}}{
}{
\input{DP_SCO_Sparse_Gradients/s7-boosting}
}

\ifthenelse{\boolean{journal}}{}{
\section{Missing Proofs and Results from \Cref{sec:Output_Perturbation}}

\label{app:missing_proofs_output_perturbation}

\subsection{Proof of \Cref{thm:output-pert-improved}}

\input{DP_SCO_Sparse_Gradients/Proofs/thm07output-pert-improved}
\subsection{Proof of \Cref{thm:output_pert_approx_DP_SCO}}

\input{DP_SCO_Sparse_Gradients/Proofs/thm09output_pert_approx_DP_SCO}
\input{DP_SCO_Sparse_Gradients/s5.1-sparse-exp-mech}
}

%% file: DP_SCO_Sparse_Gradients/Proofs/lem02sparse_codes.tex
\begin{lemma}\label[lemma]{lem:sparse_codes}
For all $s$ and $d$ such that $s \le d/2$, there exists a subset ${\cal P} \subseteq {\cal C}_s^d$ such that $|{\cal P}| \ge (d/s - 1/2)^{s/2}$ and for all $u, v \in {\cal P}$, it holds that $\|u - v\|_2 \ge 1/\sqrt{2}$.
\end{lemma}
\begin{proof}
This follows from a simple packing-based construction (see, e.g., \cite{Candes:2011}). There are $\binom{d}{s}$ vectors in $\cC_{s}^d$, and for each vector $v \in \cC_s^d$, there are at most $\binom{d}{\lfloor s/2 \rfloor}$ many vectors $u \in \cC_s^d$ such that $\|u - v\|_0 \le s/2$ and hence $\|u - v\|_2 \le 1/\sqrt{2}$. Thus, we can greedily pick vectors to be $C$, guaranteeing that all vectors $u, v \in \cC_s^d$ satisfy $\|u - v\|_0 > s/2$, and have $|C| \ge \binom{d}{s} / \binom{d}{\lfloor s/2 \rfloor} \ge \left( \frac ds - \frac12 \right)^{s/2}$.
\end{proof}

%% file: DP_SCO_Sparse_Gradients/Proofs/lem03bootstrap-privacy-lb.tex
\begin{lemma}[Dataset bootstrapping argument from \cite{Bassily:2014}]\label[lemma]{lem:bootstrap-privacy-lb}
Suppose for some $n$, there exists a mechanism $\cA$ such that for all $S\in (\sparsevec)^{n}$, it holds with probability at least $1/2$ that $\|\cA(S) - \bar z(S)\|_2 \le C$, for some $C \ge 0$.
Then for all $n^* < n$, there exists a mechanism $\cA'$ such that for all $S' \in(\sparsevec)^{n^*}$, it holds with probability at least $1/2$ that $\|\cA(S') - \bar z(S')\|_2 \le C \frac{n}{n^*}$. Furthermore, $\cA'$ satisfies the same privacy guarantees as $\cA$, namely if $\cA$ is $\eps$-DP (or $(\eps, \delta)$-DP), then so is $\cA'$.
\end{lemma}

\begin{proof}
Given mechanism $\cA$, consider mechanism $\cA'$ that for any dataset $S' \in (\sparsevec)^{n^*}$,  builds dataset $S$ by adding $n-n^*$ copies of $\mathbf{0}$ to $S'$ and returns $\frac{n}{n^*} \cA(S)$.
From the guarantees of $\cA$, it holds that
$\mathbb{P}\sq{\|\cA(S) - \bar z(S)\|_2 \le C} \ge \frac12$.
Since $\cA'(S') = \frac{n}{n^*}\cA(S)$ and $\bar z(S') = \frac{n}{n^*} \bar z(S)$, it follows that
\[
\mathbb{P}\sq{\|\cA'(S') - \bar z(S')\|_2 \le C\frac{n}{n^*}} ~\ge~ \frac12\,.
\]
Since $\cA'$ just applies $\cA$ once, it follows that $\cA'$ satisfies the same privacy guarantee as $\cA$.
\end{proof}

%% file: DP_SCO_Sparse_Gradients/Proofs/lem04packing_DP_LB.tex
Next we provide a generic reduction of existence of packing sets with pure-DP mean estimation lower bounds. Note however that the lower bounds we state work on the distributional sense.

\begin{lemma}[Packing-based mean estimation lower bound, adapted from \citet{Hardt:2010,Bassily:2014}]\label[lemma]{lem:packing_DP_LB}
Let ${\cal P}\subseteq \mathbb{R}^d$ be an $\alpha_0$-packing set of vectors with $|{\cal P}|=p$. Then, there exists a distribution $\mu$ over ${\cal P}^n$ that induces an $(\alpha,\rho)$-distributional lower bound for $\varepsilon$-DP mean estimation with $\alpha=\frac{\alpha_0}{2}\min\Big\{1,\frac{\log(p/2)}{\varepsilon n}\Big\}$ and $\rho=1/2$.
\end{lemma}
\begin{proof}
Let $n^*=\frac{\log (p/2)}{\eps}$. 
First, consider the case where $n < n^*$. We construct $p$ datasets $S_1, \ldots S_p$ where $S_l$ consists of $n$ copies of $z_l$, and define $\mu=\mbox{Unif}(\{S_1,\ldots,S_p\})$. Note that for all $k \ne l$, it holds that $\|\bar z(S_k) - \bar z(S_l)\|_2 \ge \alpha_0$. 
Suppose $\mu$ does not induce a distributional lower bound. Then there exists ${\cal A}$ which is $\varepsilon$-DP and has $\ell_2$-accuracy better than $\alpha_0/2$ w.p.~at least $1/2$: this implies in particular that
\[
\mathbb{P}_{l\sim\mbox{\footnotesize Unif}([p])}\sq{\cA(S_l) \in {\cal B}_2^d(z_l), {\textstyle \frac{\alpha_0}{2}})} \ge \frac{1}{2}. 
\]
For all distinct $k$, $l$, the datasets $S_k$ and $S_l$ differ in all $n$ entries, and hence for any $\eps$-DP mechanism $\cA$, it holds that $\mathbb{P}[\cA(S_l) \in {\cal B}_2^d(z_k, \frac{\alpha_0}{2})] \ge \frac12 e^{-\eps n}$. However, by construction, ${\cal B}_2^d(z_l, \frac{\alpha_0}{2})$ are pairwise disjoint. Hence,
\begin{align*}
1 &\geq \sum_{k=1}^p \mathbb{P}_{S\sim \mu}[{\cal A}(S)\in {\cal B}_2^d(z_k,\alpha_0/2)]
= \sum_{j=1}^p \sum_{k=1}^p\mathbb{P}_{S\sim \mu}[{\cal A}(S)\in {\cal B}_2^d(z_k,\alpha_0/2)| S=z_j] \frac1p\\
&\geq \frac{e^{-\varepsilon n}}{p} \sum_{j=1}^p \sum_{k=1}^p\mathbb{P}_{S\sim \mu}[{\cal A}(S)\in {\cal B}_2^d(z_k,\alpha_0/2)| S=z_k] \geq \frac{e^{-\varepsilon n}p}{2}.
\end{align*}
Thus, we get that $n \geq \frac{\log (p/2)}{\eps}$, which is a contradiction since we assumed $n < n^*$. Hence, $\mu$ induces an $(\alpha_0/2,1/2)$-distributional lower bound for $\varepsilon$-DP mean estimation. \\
Next, consider the case where $n > n^*$. Then the previous argument together with \Cref{lem:bootstrap-privacy-lb} implies an $(\alpha,\rho)$-lower bounded, where $\alpha=\frac{n^*}{2n}$ and $\rho=1/2$, as desired.
\end{proof}

%% file: main.bbl
\begin{thebibliography}{47}
\providecommand{\natexlab}[1]{#1}
\providecommand{\url}[1]{\texttt{#1}}
\expandafter\ifx\csname urlstyle\endcsname\relax
  \providecommand{\doi}[1]{doi: #1}\else
  \providecommand{\doi}{doi: \begingroup \urlstyle{rm}\Url}\fi

\bibitem[Amirkhanyan and Fontaine(2022)]{cloudtpu}
G.~Amirkhanyan and B.~Fontaine.
\newblock Building large scale recommenders using cloud {TPUs}, 2022.
\newblock
  \url{https://cloud.google.com/blog/topics/developers-practitioners/building-large-scale-recommenders-using-cloud-tpus}.

\bibitem[Arora et~al.(2023)Arora, Bassily, Gonz{\'a}lez, Guzm{\'a}n, Menart,
  and Ullah]{Arora:2023}
R.~Arora, R.~Bassily, T.~Gonz{\'a}lez, C.~A. Guzm{\'a}n, M.~Menart, and
  E.~Ullah.
\newblock Faster rates of convergence to stationary points in differentially
  private optimization.
\newblock In \emph{ICML}, pages 1060--1092, 2023.

\bibitem[Asi et~al.(2021{\natexlab{a}})Asi, Carmon, Jambulapati, Jin, and
  Sidford]{Asi:2021bias}
H.~Asi, Y.~Carmon, A.~Jambulapati, Y.~Jin, and A.~Sidford.
\newblock Stochastic bias-reduced gradient methods.
\newblock In \emph{NeurIPS}, 2021{\natexlab{a}}.

\bibitem[Asi et~al.(2021{\natexlab{b}})Asi, Feldman, Koren, and
  Talwar]{Asi:2021}
H.~Asi, V.~Feldman, T.~Koren, and K.~Talwar.
\newblock Private stochastic convex optimization: Optimal rates in {L1}
  geometry.
\newblock In \emph{ICML}, pages 393--403, 2021{\natexlab{b}}.

\bibitem[Asi et~al.(2021{\natexlab{c}})Asi, L{\'e}vy, and
  Duchi]{Asi:2021adapting}
H.~Asi, D.~L{\'e}vy, and J.~C. Duchi.
\newblock Adapting to function difficulty and growth conditions in private
  optimization.
\newblock In \emph{NeurIPS}, pages 19069--19081, 2021{\natexlab{c}}.

\bibitem[Bassily et~al.(2014)Bassily, Smith, and Thakurta]{Bassily:2014}
R.~Bassily, A.~D. Smith, and A.~Thakurta.
\newblock Private empirical risk minimization: Efficient algorithms and tight
  error bounds.
\newblock In \emph{FOCS}, pages 464--473, 2014.

\bibitem[Bassily et~al.(2019)Bassily, Feldman, Talwar, and
  Guha~Thakurta]{Bassily:2019}
R.~Bassily, V.~Feldman, K.~Talwar, and A.~Guha~Thakurta.
\newblock Private stochastic convex optimization with optimal rates.
\newblock In \emph{NeurIPS}, 2019.

\bibitem[Bassily et~al.(2021)Bassily, Guzm{\'a}n, and Nandi]{Bassily:2021}
R.~Bassily, C.~Guzm{\'a}n, and A.~Nandi.
\newblock Non-{E}uclidean differentially private stochastic convex
  optimization.
\newblock In \emph{COLT}, pages 474--499, 2021.

\bibitem[Blanchet and Glynn(2015)]{Blanchet:2015}
J.~H. Blanchet and P.~W. Glynn.
\newblock Unbiased monte carlo for optimization and functions of expectations
  via multi-level randomization.
\newblock In \emph{WSC}, pages 3656--3667, 2015.

\bibitem[Bousquet and Elisseeff(2002)]{Bousquet:2002}
O.~Bousquet and A.~Elisseeff.
\newblock Stability and generalization.
\newblock \emph{JMLR}, 2:\penalty0 499--526, 2002.

\bibitem[Bousquet et~al.(2020)Bousquet, Klochkov, and
  Zhivotovskiy]{Bousquet:2020}
O.~Bousquet, Y.~Klochkov, and N.~Zhivotovskiy.
\newblock Sharper bounds for uniformly stable algorithms.
\newblock In \emph{COLT}, pages 610--626, 2020.

\bibitem[Bun et~al.(2014)Bun, Ullman, and Vadhan]{Bun:2014}
M.~Bun, J.~R. Ullman, and S.~P. Vadhan.
\newblock Fingerprinting codes and the price of approximate differential
  privacy.
\newblock In \emph{STOC}, pages 1--10, 2014.

\bibitem[Bun et~al.(2017)Bun, Steinke, and Ullman]{Bun:2017}
M.~Bun, T.~Steinke, and J.~R. Ullman.
\newblock Make up your mind: The price of online queries in differential
  privacy.
\newblock In \emph{SODA}, pages 1306--1325, 2017.

\bibitem[Cai et~al.(2020)Cai, Wang, and Zhang]{Cai:2020}
T.~T. Cai, Y.~Wang, and L.~Zhang.
\newblock The cost of privacy in generalized linear models: Algorithms and
  minimax lower bounds.
\newblock \emph{arXiv preprint arXiv:2011.03900}, 2020.

\bibitem[Cai et~al.(2021)Cai, Wang, and Zhang]{Cai:2021}
T.~T. Cai, Y.~Wang, and L.~Zhang.
\newblock {The cost of privacy: Optimal rates of convergence for parameter
  estimation with differential privacy}.
\newblock \emph{Ann. Stat.}, 49\penalty0 (5):\penalty0 2825 -- 2850, 2021.

\bibitem[Candes and Tao(2005)]{Candes:2005}
E.~Candes and T.~Tao.
\newblock Decoding by linear programming.
\newblock \emph{TOIT}, 51\penalty0 (12):\penalty0 4203--4215, 2005.

\bibitem[Cand{\`e}s and Davenport(2013)]{Candes:2011}
E.~J. Cand{\`e}s and M.~A. Davenport.
\newblock How well can we estimate a sparse vector?
\newblock \emph{Applied and Computational Harmonic Analysis}, 34\penalty0
  (2):\penalty0 317--323, 2013.

\bibitem[Chaudhuri et~al.(2011)Chaudhuri, Monteleoni, and
  Sarwate]{ChaudhuriMS11}
K.~Chaudhuri, C.~Monteleoni, and A.~D. Sarwate.
\newblock Differentially private empirical risk minimization.
\newblock \emph{JMLR}, 12:\penalty0 1069--1109, 2011.

\bibitem[Dwork and Roth(2014)]{Dwork:2014}
C.~Dwork and A.~Roth.
\newblock The algorithmic foundations of differential privacy.
\newblock \emph{Found. Trends Theor. Comput. Sci.}, 9\penalty0 (3-4):\penalty0
  211--407, 2014.

\bibitem[Feldman et~al.(2020)Feldman, Koren, and Talwar]{Feldman:2020}
V.~Feldman, T.~Koren, and K.~Talwar.
\newblock Private stochastic convex optimization: optimal rates in linear time.
\newblock In \emph{STOC}, pages 439--449, 2020.

\bibitem[Ghazi et~al.(2023)Ghazi, Huang, Kamath, Kumar, Manurangsi, Sinha, and
  Zhang]{Ghazi:2023}
B.~Ghazi, Y.~Huang, P.~Kamath, R.~Kumar, P.~Manurangsi, A.~Sinha, and C.~Zhang.
\newblock Sparsity-preserving differentially private training of large
  embedding models.
\newblock In \emph{NeurIPS}, 2023.

\bibitem[Gunny et~al.(2019)Gunny, Garg, Dolgovs, and
  Subramaniam]{nvidia2019accelerating}
A.~Gunny, C.~Garg, L.~Dolgovs, and A.~Subramaniam.
\newblock Accelerating wide \& deep recommender inference on {GPUs}, 2019.
\newblock \url
  {https://developer.nvidia.com/blog/accelerating-wide-deep-recommender-inference-on-gpus}.

\bibitem[Hardt and Talwar(2010)]{Hardt:2010}
M.~Hardt and K.~Talwar.
\newblock On the geometry of differential privacy.
\newblock In \emph{STOC}, pages 705--714, 2010.

\bibitem[Jain and Thakurta(2014)]{Jain:2014}
P.~Jain and A.~G. Thakurta.
\newblock (near) dimension independent risk bounds for differentially private
  learning.
\newblock In \emph{ICML}, pages 476--484, 2014.

\bibitem[Jouppi et~al.(2023)Jouppi, Kurian, Li, Ma, Nagarajan, Nai, Patil,
  Subramanian, Swing, Towles, et~al.]{jouppi2023tpu}
N.~Jouppi, G.~Kurian, S.~Li, P.~Ma, R.~Nagarajan, L.~Nai, N.~Patil,
  S.~Subramanian, A.~Swing, B.~Towles, et~al.
\newblock {TPU v4:} an optically reconfigurable supercomputer for machine
  learning with hardware support for embeddings.
\newblock In \emph{ISCA}, 2023.

\bibitem[Kairouz et~al.(2021)Kairouz, Diaz, Rush, and Thakurta]{Kairouz:2021}
P.~Kairouz, M.~R. Diaz, K.~Rush, and A.~Thakurta.
\newblock {(Nearly) Dimension Independent Private ERM with AdaGrad Rates via
  Publicly Estimated Subspaces}.
\newblock In \emph{COLT}, pages 2717--2746, 2021.

\bibitem[Kamath and Ullman(2020)]{Kamath:2020}
G.~Kamath and J.~R. Ullman.
\newblock A primer on private statistics.
\newblock \emph{CoRR}, abs/2005.00010, 2020.
\newblock URL \url{https://arxiv.org/abs/2005.00010}.

\bibitem[Kamath et~al.(2023)Kamath, Mouzakis, Regehr, Singhal, Steinke, and
  Ullman]{Kamath:2023}
G.~Kamath, A.~Mouzakis, M.~Regehr, V.~Singhal, T.~Steinke, and J.~Ullman.
\newblock A bias-variance-privacy trilemma for statistical estimation, 2023.

\bibitem[Lee et~al.(2024)Lee, Liu, and Lu]{Lee:2024}
Y.~T. Lee, D.~Liu, and Z.~Lu.
\newblock The power of sampling: Dimension-free risk bounds in private erm,
  2024.
\newblock URL \url{https://arxiv.org/abs/2105.13637}.

\bibitem[Li et~al.(2022)Li, Liu, Hashimoto, Inan, Kulkarni, Lee, and
  Thakurta]{Li:2022}
X.~Li, D.~Liu, T.~Hashimoto, H.~Inan, J.~J. Kulkarni, Y.~T. Lee, and A.~G.
  Thakurta.
\newblock When does differentially private learning not suffer in high
  dimensions?
\newblock In \emph{NeurIPS'22}, November 2022.

\bibitem[Liu and Talwar(2019)]{Liu:2019}
J.~Liu and K.~Talwar.
\newblock Private selection from private candidates.
\newblock In \emph{STOC}, pages 298--309, 2019.

\bibitem[Ma et~al.(2022)Ma, Marinov, and Zhang]{Ma:2022}
Y.-A. Ma, T.~V. Marinov, and T.~Zhang.
\newblock Dimension independent generalization of dp-sgd for overparameterized
  smooth convex optimization.
\newblock \emph{arXiv preprint arXiv:2206.01836}, 2022.

\bibitem[McSherry and Talwar(2007)]{McSherry:2007}
F.~McSherry and K.~Talwar.
\newblock Mechanism design via differential privacy.
\newblock In \emph{FOCS}, pages 94--103, 2007.

\bibitem[Nikolov and Tang(2023)]{Nikolov:2023}
A.~Nikolov and H.~Tang.
\newblock Gaussian noise is nearly instance optimal for private unbiased mean
  estimation.
\newblock \emph{arXiv preprint arXiv:2301.13850}, 2023.

\bibitem[Nikolov et~al.(2013)Nikolov, Talwar, and Zhang]{Nikolov:2013}
A.~Nikolov, K.~Talwar, and L.~Zhang.
\newblock The geometry of differential privacy: the sparse and approximate
  cases.
\newblock In \emph{STOC}, pages 351--360, 2013.

\bibitem[Pisier(1981)]{Pisier:1981}
G.~Pisier.
\newblock Remarques sur un r{\'e}sultat non publi{\'e} de b. maurey.
\newblock \emph{S{\'e}minaire Analyse fonctionnelle (dit}, pages 1--12, 1981.

\bibitem[Rosenthal(2006)]{Rosenthal:2006}
J.~S. Rosenthal.
\newblock \emph{A first look at rigorous probability theory}.
\newblock World Scientific Publishing Co. Pte. Ltd., Hackensack, NJ, second
  edition, 2006.
\newblock ISBN 978-981-270-371-2; 981-270-371-3.

\bibitem[Shalev{-}Shwartz et~al.(2009)Shalev{-}Shwartz, Shamir, Srebro, and
  Sridharan]{SSSS09}
S.~Shalev{-}Shwartz, O.~Shamir, N.~Srebro, and K.~Sridharan.
\newblock Stochastic convex optimization.
\newblock In \emph{COLT}, 2009.

\bibitem[Steinke and Ullman(2016)]{Steinke:2016}
T.~Steinke and J.~R. Ullman.
\newblock Between pure and approximate differential privacy.
\newblock \emph{J. Priv. Confidentiality}, 7\penalty0 (2), 2016.

\bibitem[Talwar et~al.(2014)Talwar, Thakurta, and Zhang]{Talwar:2014}
K.~Talwar, A.~Thakurta, and L.~Zhang.
\newblock Private empirical risk minimization beyond the worst case: The effect
  of the constraint set geometry.
\newblock \emph{arXiv preprint arXiv:1411.5417}, 2014.

\bibitem[Talwar et~al.(2015)Talwar, Thakurta, and Zhang]{Talwar:2015}
K.~Talwar, A.~Thakurta, and L.~Zhang.
\newblock Nearly-optimal private {LASSO}.
\newblock In \emph{NIPS}, pages 3025--3033, 2015.

\bibitem[Vershynin(2009)]{Vershynin:2009}
R.~Vershynin.
\newblock On the role of sparsity in compressed sensing and random matrix
  theory, 2009.

\bibitem[Wang et~al.(2022)Wang, Wei, Lee, Langer, Yu, Liu, Liu, Abel, Guo,
  Dong, et~al.]{wang2022merlin}
Z.~Wang, Y.~Wei, M.~Lee, M.~Langer, F.~Yu, J.~Liu, S.~Liu, D.~G. Abel, X.~Guo,
  J.~Dong, et~al.
\newblock Merlin {hugeCTR}: {GPU}-accelerated recommender system training and
  inference.
\newblock In \emph{RecSys}, 2022.

\bibitem[Whitehouse et~al.(2023)Whitehouse, Ramdas, Rogers, and
  Wu]{Whitehouse:2023}
J.~Whitehouse, A.~Ramdas, R.~Rogers, and S.~Wu.
\newblock Fully-adaptive composition in differential privacy.
\newblock In \emph{ICML}, pages 36990--37007, 2023.

\bibitem[Wojtaszczyk(2010)]{Wojtaszczyk:2010}
P.~Wojtaszczyk.
\newblock Stability and instance optimality for {Gaussian} measurements in
  compressed sensing.
\newblock \emph{Foundations of Computational Mathematics}, 10:\penalty0 1--13,
  2010.

\bibitem[Zhang et~al.(2021)Zhang, Mironov, and Hejazinia]{Zhang:2021}
H.~Zhang, I.~Mironov, and M.~Hejazinia.
\newblock Wide network learning with differential privacy.
\newblock \emph{CoRR}, abs/2103.01294, 2021.

\bibitem[Zhou et~al.(2021)Zhou, Wu, and Banerjee]{Zhou:2021}
Y.~Zhou, S.~Wu, and A.~Banerjee.
\newblock Bypassing the ambient dimension: Private {SGD} with gradient subspace
  identification.
\newblock In \emph{ICLR}, 2021.

\end{thebibliography}
